\PassOptionsToPackage{unicode}{hyperref}
\PassOptionsToPackage{hyphens}{url}
\PassOptionsToPackage{dvipsnames,svgnames,x11names}{xcolor}
\documentclass[
  12pt]{article}
\usepackage{siunitx}

\usepackage{algpseudocode}
\usepackage{amsthm}  
\usepackage{algorithm}
\newtheorem{theorem}{Theorem}[section] 
\usepackage{algpseudocode}
\usepackage{amsmath,amssymb}
\usepackage{iftex}
\usepackage{etoolbox}
\makeatletter
\patchcmd\longtable{\par}{\if@noskipsec\mbox{}\fi\par}{}{}
\makeatother
\IfFileExists{footnotehyper.sty}{\usepackage{footnotehyper}}{\usepackage{foo\ifPDFTeX
  \usepackage[T1]{fontenc}
  \usepackage[utf8]{inputenc}
  \usepackage{textcomp} 
\else 
  \usepackage{unicode-math}
  \defaultfontfeatures{Scale=MatchLowercase}
  \defaultfontfeatures[\rmfamily]{Ligatures=TeX,Scale=1}
  
\fi
\usepackage{lmodern}
}}
\ifPDFTeX\else  
\fi
\IfFileExists{upquote.sty}{\usepackage{upquote}}{}
\IfFileExists{microtype.sty}{
  \usepackage[]{microtype}
  \UseMicrotypeSet[protrusion]{basicmath} 
}{}
\makeatletter
\@ifundefined{KOMAClassName}{
  \IfFileExists{parskip.sty}{%
    \usepackage{parskip}
  }{
    \setlength{\parindent}{0pt}
    \setlength{\parskip}{6pt plus 2pt minus 1pt}}
}{
  \KOMAoptions{parskip=half}}
\makeatother
\usepackage{xcolor}
\makeatletter
\ifx\paragraph\undefined\else
  \let\oldparagraph\paragraph
  \renewcommand{\paragraph}{
    \@ifstar
      \xxxParagraphStar
      \xxxParagraphNoStar
  }
  \newcommand{\xxxParagraphStar}[1]{\oldparagraph*{#1}\mbox{}}
  \newcommand{\xxxParagraphNoStar}[1]{\oldparagraph{#1}\mbox{}}
\fi
\ifx\subparagraph\undefined\else
  \let\oldsubparagraph\subparagraph
  \renewcommand{\subparagraph}{
    \@ifstar
      \xxxSubParagraphStar
      \xxxSubParagraphNoStar
  }
  \newcommand{\xxxSubParagraphStar}[1]{\oldsubparagraph*{#1}\mbox{}}
  \newcommand{\xxxSubParagraphNoStar}[1]{\oldsubparagraph{#1}\mbox{}}
\fi
\makeatother

\usepackage{longtable,booktabs,array}
\usepackage{calc} 
\makesavenoteenv{longtable}
\usepackage{graphicx}
\makeatletter
\def\maxwidth{\ifdim\Gin@nat@width>\linewidth\linewidth\else\Gin@nat@width\fi}
\def\maxheight{\ifdim\Gin@nat@height>\textheight\textheight\else\Gin@nat@height\fi}
\makeatother
\setkeys{Gin}{width=\maxwidth,height=\maxheight,keepaspectratio}
\makeatletter
\def\fps@figure{htbp}
\makeatother

\makeatletter
\@ifpackageloaded{caption}{}{\usepackage{caption}}
\AtBeginDocument{%
\ifdefined\contentsname
  \renewcommand*\contentsname{Table of contents}
\else
  \newcommand\contentsname{Table of contents}
\fi
\ifdefined\listfigurename
  \renewcommand*\listfigurename{List of Figures}
\else
  \newcommand\listfigurename{List of Figures}
\fi
\ifdefined\listtablename
  \renewcommand*\listtablename{List of Tables}
\else
  \newcommand\listtablename{List of Tables}
\fi
\ifdefined\figurename
  \renewcommand*\figurename{Figure}
\else
  \newcommand\figurename{Figure}
\fi
\ifdefined\tablename
  \renewcommand*\tablename{Table}
\else
  \newcommand\tablename{Table}
\fi
}
\@ifpackageloaded{float}{}{\usepackage{float}}
\floatstyle{ruled}
\@ifundefined{c@chapter}{\newfloat{codelisting}{h}{lop}}{\newfloat{codelisting}{h}{lop}[chapter]}
\floatname{codelisting}{Listing}

\makeatother
\makeatletter
\@ifpackageloaded{caption}{}{\usepackage{caption}}
\@ifpackageloaded{subcaption}{}{\usepackage{subcaption}}
\makeatother

\ifLuaTeX
  \usepackage{selnolig}  
\fi
\usepackage[]{natbib}
\usepackage{bookmark}

\IfFileExists{xurl.sty}{\usepackage{xurl}}{} 
\hypersetup{
  pdftitle={Title},
  pdfauthor={Author 1; Author 2},
  pdfkeywords={3 to 6 keywords, that do not appear in the title},
  colorlinks=true,
  linkcolor={blue},
  filecolor={Maroon},
  citecolor={Blue},
  urlcolor={Blue},
  pdfcreator={LaTeX via pandoc}}

\newcommand{\anon}{1}

\begin{document}

\def\spacingset#1{\renewcommand{\baselinestretch}%
{#1}\small\normalsize} \spacingset{1}


\if1\anon
{
  \title{\bf Proximal Projection for Doubly Sparse Regularized Models}
  \author{Jia Wei He\\
    Department of Mathematics \& Statistics
, University of Guelph\\
    and \\
    Ayesha Ali \\
    Department of Mathematics \& Statistics
, University of Guelph\\
and \\
    Gerarda Darlington \\
    Department of Mathematics \& Statistics
, University of Guelph
\\and \\
for the Alzheimer’s Disease Neuroimaging Initiative\footnote{Data used in preparation of this article were obtained from the Alzheimer’s Disease Neuroimaging Initiative
(ADNI) database (adni.loni.usc.edu). As such, the investigators within the ADNI contributed to the design
and implementation of ADNI and/or provided data but did not participate in analysis or writing of this report.
A complete listing of ADNI investigators can be found at:
\url{http://adni.loni.usc.edu/wp-content/uploads/how_to_apply/ADNI_Acknowledgement_List.pdf}}\\}
  \maketitle
} \fi

\if0\anon
{
  \bigskip
  \bigskip
  \bigskip
  \begin{center}
    {\LARGE\bf Title}
\end{center}
  \medskip
} \fi

\bigskip
\begin{abstract}

Regularization is often used in high-dimensional regression settings to generate a sparse model, which can save tremendous computing resources and identify predictors that are most strongly associated with the response. When the predictors can be represented by a Gaussian graphical model, the structure of the predictor graph can be exploited during regularization. Our proposed model exploits this underlying predictor graph structure by decomposing the estimated coefficient vector into a sum of latent variables that correspond to the sum of each node’s contribution to the coefficient vector. Regularization is then performed on the latent variables rather than on the coefficient vector directly. We use a penalty function that permits a clear user-defined trade-off between the $l_1$ and $l_2$ penalties and propose a novel proximal projection during optimization. Further, our implementation computes the projection operator for the intersection of selected groups, which conserves more computing resources compared to predictor duplication methods, especially for high-dimensional data. Through simulation, we evaluate the performance of our approach under different graph structures and node counts, and present results on real-world data. Results suggest that our method exhibits stable performance relative to other singly or doubly sparse graphical regression models.
\end{abstract}

\noindent%
 {\it Keywords:} Graphical model, Proximal methods, Doubly regularized LASSO

\spacingset{1.8} 

\section{Introduction}
\label{sec:introduction}
High dimensional statistics have been widely studied in recent decades with the rise of the size of datasets collected, both in dimensions and sample size. In biotechnology, \citet{oliveira2019biotechnology} highlighted the exponential growth in the generation of nucleotide and proteomics data, which dramatically increased the number of available biological predictors. Similarly, as the financial market develops, more financial information is constantly being collected. \citet{ando2017clustering} analyzed the return of more than 6,000 stocks in more than 100 financial markets, showcasing the large scale of financial data now available for analysis. Linear regression is known to perform well with high-dimensional data, and ordinary least squares (OLS) is a common method for parameter estimation, particularly when sample size $n$ is larger than dimension $p$. 

However, when $p$ is greater than $n$, the design matrix is not of full rank and a unique solution does not exist. Similarly, if any predictors are highly correlated then multicollinearity issues may arise, potentially giving misleading results. Thus, linear regression models with sparsity have become popular for processing high-dimensional data with correlated features. Consider the LASSO \citep{tibshirani1997lasso}, which employs an  $l_1$  penalty to regularize the model. LASSO yields sparse solutions by forcing some coefficients to be exactly zero, effectively performing variable selection. On the other hand, ridge regression \citep{hoerl1970ridge} applies an $l_2$ penalty, which shrinks the magnitude of the coefficients without setting them exactly to zero. Ridge regression provides a better bias-variance trade-off, producing more stable estimates with lower bias. 

The success of the LASSO has inspired numerous extensions focused on variable selection. Notably, in 2006, the concept of structural variable selection was introduced with the development of the group LASSO \citep{yuan2006model}. Variables in the design matrix are partitioned into groups, and the penalty term is applied to the entire group such that the model will either select or eliminate an entire group of variables, rather than individual ones. Variables in the group LASSO are classified into sets of non-overlapping groups, though the group LASSO for overlapping groups was introduced by \citet{jacob2009group}. Further,  \citet{obozinski2011group} proposed the latent group LASSO approach with overlapped groups. LASSO is also widely applied in other fields. For example, \citet{meinshausen2006high} and \citet{friedman2008sparse} implemented the LASSO approach in estimating sparse graph structures for Gaussian graphical models. 

Building on this, \citet{yu2016sparse} used an undirected graph to represent structure information among the predictors of a regression model and combined it with the latent group LASSO approach. Such approaches could be applied in various fields. For example, \citet{hu2015spectral} studied structural brain connectivity of Alzheimer's disease (AD) using an undirected graph regression model. \citet{stephenson2019dsrig} extended this framework by introducing an $l_1$ penalty to both between and within group regularization to the latent variables. Both models exploit the underlying structure associated with the predictor graph by (a) decomposing the estimated coefficient vector into a sum of latent variables, corresponding to the sum of each node's contribution to the coefficient vector, and (b) performing regularization on the latent variables rather than on the coefficient vector directly. However, \citet{stephenson2019dsrig} used predictor duplication \citep{obozinski2011group} to separate the overlapped groups, which demands high computing resources if the dimension of the dataset is high or the predictor graph is dense. 

The objectives of this paper are as follows:
\begin{enumerate}
    \item Develop an algorithm to compute the proximal projection for the regularized model with double sparsity and prove its validity;
    \item Introduce a Sparse overlapping Group LASSO Incorporating Graphical structure (SGLIG) model that improves efficiency while maintaining double sparsity;
    \item Derive a finite sample error bound for the SGLIG model, assuming the design matrix follows a multivariate normal distribution; and
    \item Evaluate the performance of the SGLIG model under different graphical structures.

\noindent Our novel SGLIG model is also demonstrated on real world data.
\end{enumerate}

This paper is organized as follows.  Section \ref{sec: Background} sets notation and provides a brief review of regularization,  undirected graphs, proximal algorithms, and regularized regression incorporating graphical structure. Section \ref{sec: Methodology} describes the derivation of the SGLIG model, the novel doubly projected proximal algorithm, a two-stage projection, and theoretical properties of the SGLIG model. Section \ref{sec: simulation} explores the performance of the SGLIG model on simulated data with various predictor graph structures. Section \ref{sec: real data} compares the performance of the SGLIG model on two real-world datasets. Section \ref{sec: conclusion} provides conclusions and directions for future work.

\section{Background}
\label{sec: Background}
\subsection{Regularized Models}
Let $X$ be an $n\times p$ matrix with $n$ observations and $p$ predictors, and assume that each $p$-dimensional observation in $X$ is independent and identically normally distributed; i.e., $X\sim \mathcal{MVN}(\mu, \Sigma)$, where $\mu$ is a $p$-dimensional vector and $\Sigma$ is a $p\times p$ covariance matrix. Let $Y$ be a univariate response variable of length $n$. We assume $\beta$ is a $p\times 1$ coefficient vector to be estimated in the following linear model,
\begin{equation}
Y = X\beta + \varepsilon,
\label{eq:linear_model}
\end{equation}
where \( \varepsilon \) is given as $\varepsilon = (\varepsilon_1, \varepsilon_2, \cdots, \varepsilon_n)$, where each \( \varepsilon_i \) represents the individual errors with $\varepsilon_i \stackrel{iid}{\sim} \mathcal{N}(0, \sigma^2) $ for $i=1, \ldots, n$. OLS proceeds by minimizing the sum of differences between observed and predicted values, given by 
\begin{equation}
    \hat{\beta} = \arg\min_{\beta} \|Y - X \beta\|^2.
    \label{eq: OLS}
\end{equation}
When $X$ is nonsingular and has full rank, the estimator $\hat{\beta} = (X^T X)^{-1} X^TY$ is easily found using matrix multiplication. However, these conditions may be violated in many practical applications, particularly when dealing with high-dimensional data where $p>>n$. In such cases, the design matrix could be singular and the inversion of $X^T X$ is not guaranteed. Regularized regression adds a penalty function to the objective function in (\ref{eq: OLS}). For instance, the LASSO is given by
\begin{equation}
  \hat{\beta}_{\text{LASSO}} = \arg\min_{\beta} \left( \| Y - X \beta \|^2 + \lambda \| \beta \|_1 \right),  
  \label{eq: LASSO}
\end{equation}
where $ \| \cdot \|_1 $ is the $l_1$ norm. In high-dimensional settings where it may be reasonable that the predictors in $X$ are associated according to an undirected (Gaussian) graphical model,  \citet{yu2016sparse} and \citet{stephenson2019dsrig} introduced regularized models that exploit this graphical structure to compute $\hat{\beta}$, the estimated regression coefficient vector.  Before reviewing their models, we briefly review undirected graphical models and proximal projections.

\subsection{Undirected Graphical Model}
A graphical model is a tuple $G = (V, E, P)$, where $V$ is the set of vertices (here, corresponding to the predictors in our predictor graph), $E$ is the set of vertex pairs $\{v_i, v_j\}, i \neq j$ connected by an edge in $G$, and $P$ is the set of distributions encompassing the sets of conditional independence relations entailed by the graph. 
When the edges in $E$ are undirected, and $P$ is a joint multivariate normal distribution, then $G$ is called a Gaussian graphical model. The neighbors of a node $v_i$ includes all the nodes that are connected to node $v_i$, denoted $ne(v_i)$, whereas the neighborhood of $v_i$ is $\mathcal{N}_i = ne(v_i) \cup v_i$. 
Further, in $P$, node $X_i$ is independent of all nodes in $\mathcal{N}_i^C$ conditional on $\mathcal{N}_i \setminus \{ X_i \}$. Thus, every missing edge in $G$ is associated with a conditional independence relation in $P$. For a Gaussian graphical model, $P$ corresponds to a multivariate normal distribution and every missing edge in the undirected graph corresponds to a zero in the precision matrix, $\Omega$, associated with $P$.  For example, if $X_3 \perp \{X_1, X_4, X_5\} \mid X_2$ in a graph with five nodes, we have that in $\Omega$, elements $\omega_{31} = \omega_{34} = \omega_{35} = 0$. 

Our SGLIG model will assume that the predictors in $X$ can be represented by a Gaussian graphical model and we will exploit this in estimating $\beta$ using a proximal algorithm.

\subsection{Proximal Algorithm}
The proximal algorithm is a method that is capable of solving convex optimization problems with non-smooth constraints by finding the proximal operator of a function rather than working with derivatives or subgradients of the function \citep{bauschke2011convex}. It is used in many fields because of its flexibility and efficiency, and has closed-form solutions for some commonly used regularized models, such as the LASSO. See \citet{parikh2014proximal} for a discussion on the theoretical properties of the proximal algorithm and applications to some classical regularizers.  

Suppose $f$ is a given closed proper convex function $
f : \mathbb{R}^n \to \mathbb{R} \cup \{+\infty\}$. The proximal operator for the function $f$ is defined by 
\begin{equation}
    \text{prox}_f(v) = \arg\min_x \left\{ f(x) + \frac{1}{2} \|x - v\|_2^2 \right\}.
    \label{eq: proximal operator}
\end{equation}
Since the above function is strongly convex and not everywhere infinite, $ \text{prox}_f(v)$ has a unique solution for every $v \in \mathbb{R}^n$. Thus, the proximal operator of a convex function $f$ evaluated at $v$ is a compromised solution that minimizes $f$ and is as close to $v$ in Euclidean distance as possible. The proximal algorithm can be thought of as a fixed-point iteration with $v^{t+1} := \text{prox}_f(v^t)$. In particular, if $f$ has a minimum, then $v^t$ converges to the set of minimizers of $f$ and $f(v^t)$ converges to its optimal value \citep{bauschke2011convex}.

The proximal operator often uses a scaled $f(x)$, such that (\ref{eq: proximal operator}) can be re-written with a scale parameter $\lambda$ as follows, 
\begin{equation}
    \text{prox}_{\lambda\mathcal{R}}(\beta) = \arg\min_\beta \left\{ \mathcal{L}(\beta,\hat{\beta}) + \lambda \mathcal{R}(\beta) \right\},
    \label{eq: generalized proximal operator}
\end{equation}
where $\mathcal{L}(\beta,\hat{\beta}) =  \frac{1}{2n} \|\hat{\beta} - \beta\|_2^2$ is a smooth loss function, and the $\lambda \mathcal{R}(\beta)$ could be regarded as a scaled $f(x)$ in equation (\ref{eq: proximal operator}).

\subsection{Regularized Regression Incorporating Graphical Structure}
Recall that for the multivariate normal distribution, the inverse of the covariance matrix can reflect the joint distribution of any two variables in $X$ given the rest of the variables \citep{lauritzen1996graphical}. Therefore, $\Sigma^{-1} = \Omega = \left( \omega_{i,j} \right)_{i,j=1}^p $, and the element $\omega_{i,j}$ reflects the conditional independence relation entailed by the graph. Let $\Sigma_{xy} = (c_1, c_2, \dots, c_p)$ represent the cross-covariance vector between the design matrix $X$ and the response variable $Y$. 

With the linear regression model, it can be easily shown that $\beta = \Omega \Sigma_{xy}$. In other words, the unknown coefficient vector $\beta$ can be decomposed as the matrix product of the inverse of the covariance matrix associated with the predictors in $X$ and cross-covariances between the predictors and the response. As such, $\beta$ can be expressed by the following series of equations:

\begin{align*}
\beta_{1} &= c_1\omega_{11} + c_2\omega_{12} + \ldots + c_i\omega_{1i} + \ldots + c_p\omega_{1p} \\
\beta_{2} &= c_1\omega_{21} + c_2\omega_{22} + \ldots + c_i\omega_{2i} + \ldots + c_p\omega_{2p} \\
&\vdots \\
\beta_{p} &= \underbrace{c_1\omega_{p1}}_\text{$V^{(1)}$} + \underbrace{c_2\omega_{p2}}_\text{$V^{(2)}$} + \ldots + \underbrace{c_i\omega_{pi}}_\text{$V^{(i)}$} + \ldots + \underbrace{c_p\omega_{pp}}_\text{$V^{(p)}$}.
\label{eq:decomposed beta}
\end{align*}
 Thus, the above equations could be reformulated such that $\beta = \Omega \Sigma_{xy} = \sum^{p}_{i=1}{V^{(i)}}$, where $V^{(i)} = \{ c_i\omega_{ji} \mid  j \in \mathcal{N}_i \}$ is the aggregate value of $c_i$ multiplied by the $i^{\text{th}}$ column of $\Omega$. Then, $V^{(i)}$ represents the contribution of predictor $i$ to the unknown coefficient vector $\beta$, and $V^{(i)}_j = c_i\omega_{ji}$ represents predictor $i$\text{'s} contribution to the $j^{\text{th}}$ element in $\beta$. Note that if $c_i = 0$, that is, the cross-covariance between $X_i$ and $Y$ is $0$, then $V^{(i)} = \{ c_i\omega_{ji} \mid  j \in \mathcal{N}_i \}$ will equal the zero vector. On the other hand, since $X$ is multivariate normally distributed, $\omega_{ji} = 0$ when predictors $i$ and $j$ are conditionally independent. Subsequently, the $j^{\text{th}}$ element in $V^{(i)}$ equals zero when $\omega_{ji} = 0$ even if $c_i \neq 0$.

\citet{yu2016sparse} introduced Sparse Regression Incorporating Graphical Structure (SRIG), in which the coefficient vector $\beta$ is decomposed into the sum of the (latent) $V^{(i)}\text{'s}$ and regularization is performed on the latent $V^{(i)}\text{'s}$, rather than on the $\beta$\text{'s} directly.  That is, the optimization function is given by,
\begin{equation}
\min_{\beta,V^{(1)},\ldots,V^{(p)}} { \frac{1}{2n} \|Y - X\beta\|^2_2 }+ \lambda   \sum_{i=1}^p \tau_i \|V^{(i)}\|_2 ,
\label{eq:SRIG}
\end{equation}
where $\beta = \sum_{i=1}^{p} V^{(i)}$, $\tau_i$ is the weight of each predictor $i$, and $supp(V^{(i)})  \subseteq \mathcal{N}_i $. The SRIG model is similar to a group LASSO model, but it operates in the latent space of the $V^{(i)}$\text{'s}.  SRIG assumes that the structure of the predictor graph is known, but in practice, the predictor graph is often estimated from the data.  Unfortunately, the group sparsity applied to the estimated predictor graph does not correct for false-positive edges, which may lead to misleading results.

DSRIG model is a doubly sparse extension of the SRIG model that is proposed to mitigate this problem by adding within-group sparsity to the model, that is 

\begin{equation}
\min_{\beta,V^{(1)},\ldots,V^{(p)}} { \frac{1}{2n} \|Y - X\beta\|^2_2 }+ \lambda  \Bigl\{  \sum_{i=1}^p \tau_i \|V^{(i) } \|_2 + \xi  \|V^{(i)}\|_1 \Bigl\},
\label{eq:DSRIG}
\end{equation}
where $\xi>0$ is a tuning parameter that balances the group and within-group sparsity. DSRIG model is similar to the sparse group LASSO model \citep{simon2013sparse}, but the penalty is added on the decomposed $\beta$ instead of directly on $\beta$. By adding  $\|V^{(i)}\|_1$ to the model, the DSRIG model shrinks individual elements in each $V_j^{(i)}$ to zero so that the DSRIG has sparsity both between $V^{(i)}$ and within $V^{(i)}$.

\section{Methodology}
\label{sec: Methodology}

The DSRIG model incorporates both between and within group sparsity, whereas the SRIG only has between group sparsity. However, the DSRIG model also has an additional tuning parameter $\xi$ compared to the SRIG model, which controls the trade-off between the $l_1$ and $l_2$ penalties. Therefore, a greedy search among the grid of ($\lambda$, $\xi$) will be needed to find the optimal solution for DSRIG. Further, DSRIG combines predictor duplication, which duplicates columns in $X$ according to the overlapping neighborhoods, with FISTA (Fast Iterative Shrinkage-Thresholding Algorithm \citep{beck2009fast}) to optimize the objective function in (\ref{eq:DSRIG}). When the predictor graph is high-dimensional or the groups are highly overlapped, optimization is time consuming. These computational challenges motivate the creation of the SGLIG model, which retains the double sparsity of DSRIG using only one tuning parameter and uses a doubly proximal algorithm for optimization.

\subsection{SGLIG}
Similar to SRIG and DSRIG, SGLIG needs a pre-defined predictor graph before optimization. The $l_2$ penalty in the DSRIG model is like an adaptive LASSO model \citep{zou2006adaptive} that has a particular level of shrinkage ($\tau_i$) for each group while the $l_1$ penalty is the same as a normal LASSO model that applies the same level of shrinkage ($\xi$) within each group. However, the neighborhood of each node may have extremely different sizes in some graph structures, and this motivates us to propose an SGLIG model that applies different shrinkage to the groups depending on the group size. The SGLIG model is a reparametrized DSRIG model that is naturally similar to the sparse group LASSO model and is given by, 
\begin{equation}
    \min_{\beta,V^{(1)},\ldots,V^{(p)}} \frac{1}{2n} \|Y - X\beta\|^2_2 +  \lambda^* \Bigl\{\sum_{i=1}^p  \alpha\times\tau_i \|V^{(i)}\|_2 + (1-\alpha)\times\sqrt{d_i} \|V^{(i)}\|_1\Bigl\},
    \label{eq:SGLIG}
\end{equation}
where $ \lambda^*$ is a constant that controls overall shrinkage, $0\leq\alpha\leq1$ specifies the trade-off between $l_2$ and $l_1$ penalization, and $d_i$ denotes the degree of each neighborhood that represents the size of the neighborhood. Thus, the shrinkage within each group also depends on the group size, and the larger the groups, the more sparsity will be encouraged. SGLIG only has one tuning parameter $\alpha$ because we denote $ \lambda^*$  as a constant instead of a tuning parameter. A possible specification of $ \lambda^*$ could be $ \lambda^* = \frac{\lambda_{max}}{c}$ where $\lambda_{max}$ is the maximum $\lambda$ that shrinks all the $V^{(i)}$ to zero and $c$ is a constant. Theoretically, $c$ could be treated as a tuning parameter; however, we fix $c$ as a constant in both simulated and real-data analyses so that SGLIG has only one tuning parameter. Specifically, we set $c = 5$, and this choice works well in both simulation and real-data analyses. We will show in Section 5 that SGLIG model is competitive to DSRIG in terms of root mean square error of estimation, but has computational speed comparable to SRIG.

We can show that the SGLIG objective function in (\ref{eq:SGLIG}) could be expressed in a more generalized and simplified form,
\begin{equation}
\min_{\beta} \mathcal{L}(\beta) + \lambda \mathcal{R}(\beta), \,\,\beta = \sum_{i=1}^{p} V^{(i)},
    \label{eq: penalized equation}
\end{equation}
where $\mathcal{L}(\beta)$ is a smooth loss function, and $\mathcal{R}(\beta)$ represents the penalized term minimized by the optimal decomposition of $\beta$. In particular, we have 
\[\mathcal{L}(\beta) = \frac{1}{2n} \|Y - X\beta\|^2_2,\]
and
\[\mathcal{R}(\beta) =   \min_{\beta = \sum{V^{(i)}}_{i = 1}^{p},supp(V^{(i)})\subseteq \mathcal{N}_i} \sum_{i=1}^p  \alpha\times\tau_i \|V^{(i)}\|_2 + (1-\alpha)\times\sqrt{d_i} \|V^{(i)}\|_1.\]
Similar works on the decomposition of $\beta$ have been done by \citet{jacob2009group, obozinski2011group, rao2013sparse, yu2016sparse, stephenson2019dsrig}, all of which have shown that the optimal decomposition always exists, but it may not be unique. 

 The $l_1$ and $l_2$ penalties in the model balance the between and within group sparsity via the tuning parameters. Since $\lambda^*$ is a constant, SGLIG is given constant total shrinkage, and then tuning the trade-off between $l_1$ and $l_2$ penalty under this total shrinkage. In other words, the level of $l_1$ and $l_2$ penalty not only depends on $\tau_i$ and $\sqrt{d_i}$ but also depends on $\alpha$ and $1-\alpha$. With this trade-off, if SGLIG has more group-wise shrinkage on the $V^{(i)}$ with a relatively high $\alpha$, then $1-\alpha$ will be small and less sparsity will be added within this group. 

\subsection{Doubly Projected Proximal Algorithm}
Our SGLIG model has some similarities with the Sparse Overlapping Sets LASSO proposed by \citet{rao2013sparse}, which also has both between and within group sparsity among the overlapped groups. They used a gradient proximal point algorithm that iteratively took gradient steps of the loss function in the negative direction and subsequently used the proximal operator with constraints to find the optimization values. We similarly propose using these proximal gradient methods; with equation (\ref{eq: penalized equation}), the proximal gradient method is 
\begin{equation}
    \beta_{k+1} := \text{prox}_{\lambda\mathcal{R}}\left(\beta_k - \lambda \nabla \mathcal{L}(\beta_k)\right),
    \label{eq: proximal gradient method}
\end{equation}
where $\lambda $ is a pre-defined step size and $\nabla \mathcal{L}(x_k) = \frac{1}{n} X^T \left( X \beta - Y \right)$. \citet{beck2009fast} developed FISTA as an extension of proximal methods with significant improvement in computational efficiency. Based on FISTA, \citet{villa2014proximal} proposed an active group strategy to solve the penalized regression model with potentially overlapped group information. We further prove that this active group strategy works for our SGLIG model and combined it with FISTA methods. Specifically, we propose the following doubly projected proximal algorithm. 

\begin{algorithm}
\caption{Doubly Projected Proximal Algorithm}\label{alg:two}
\begin{algorithmic}[1]  

\State \textbf{Given: } $\mathbb{G}, p \in (1, +\infty], \tau > 0, \epsilon_{0} > 0,  \beta^{0} = h^{1} \in \mathbb{R}^p, t_{1} = 1$ 
\State \textbf{Let: } $\sigma = \lVert X^T X \rVert/n, m = 1 $ and $q$ such that $\frac{1}{p} + \frac{1}{q} = 1$
\While{\textnormal{convergence not reached}}
    \State ${h}^{m} = Z^{m} - \frac{1}{n\sigma}X^T(X Z^m - y)$  
    \State \textbf{Find} $\hat{\mathbb{G}}^m = \{ i :  \lVert {h^m_{\mathcal{N}_i}} \rVert_{2} \geq \tau_i^* \text{ or } \|{h^m_{\mathcal{N}_i}} \|_{\infty} \geq \xi_i^*\}$ 
    
    \State $\mathbf{K}^{\hat{\mathbb{G}}^m}_{p} := \{ \beta \in \mathbb{R}^p : \lVert \beta_{\mathcal{N}_i}\rVert_{2} \leq \tau_i^* \text{ and } \lVert \beta_{\mathcal{N}_i}\rVert_{\infty} \leq \xi_i^* \text{ for each } i \in \hat{\mathbb{G}}^m\}$     
    \State $\beta^{m} = h^{m} - \arg\min_{\beta \in \mathbf{K}^{\hat{\mathbb{G}}^m}_{p}} \lVert \beta - h^m \rVert$ 
    \State $t_{m+1} = \frac{1+\sqrt{1+4t^2_{m}}}{2}$ 
    \State $Z^{m+1} = \beta^{m} + \frac{t_m-1}{t_{m+1}} (\beta^m - \beta^{m-1})$ 
\EndWhile
\State \Return{$\beta^m$}

\end{algorithmic}
\end{algorithm}
The $\tau_i^*$ and $\xi_i^*$ are constants controlling the shrinkage between and within group sparsity. The doubly projected proximal algorithm is developed for a regularized regression model with double sparsity, evaluating both the $l_2$ and $l_\infty$ norms for each neighborhood in step 5. In step 6, $\mathbf{K}^{\hat{\mathbb{G}}^m}_{p}$ is the intersection of convex sets, where each convex set is the intersection of $l_2$ and $l_\infty$ norm balls. Appendix A.2 derives the theoretical details needed to derive this algorithm. 

For step 7, we provide details in the Appendix to show that the computation of the proximal operator of overlapped groups is equivalent to calculating the proximal projection onto the intersection of convex sets.  Our SGLIG model aims to find a parameter that balances the trade-off between the $l_1$ and $l_2$ penalties, thus the size of selected active groups in step 5 may be large for some particular choices of parameter, leading to increased computational time in step 7. As proved by \citet{beck2009fast}, step 7 converges to the optimization solution with rate $O\left(\frac{1}{m^2}\right)$. However, the computation of step 7 is challenging since it calculates the projection onto the intersection of convex sets. There are multiple computation methods proposed to solve this problem. More details are discussed in the following subsection. 

\subsection{Proximal Projection onto Intersection of Convex Sets}
\citet{combettes2011proximal} proposed the Parallel Dykstra-like proximal algorithm that iteratively finds the average of the proximal operator projected onto each convex set. \citet{yu2016sparse} used this method when the size of the active group $\hat{\mathbb{G}}^m$ is greater than $\frac{p}{10}$. If the number of active groups is smaller than $\frac{p}{10}$, \citet{yu2016sparse} calculate the projection by finding the dual problem using Bertsekas's projected Newton method proposed by \citet{villa2014proximal}. In practice, for highly dimensional graphs, the computational time increases as the density of the graph increases, particularly when fitting models with double sparsity. However, we use the Parallel Dykstra-like proximal algorithm \citep{combettes2011proximal} to calculate the projection onto convex sets regardless of the size of $\hat{\mathbb{G}}^m$; see Parallel Dykstra-like proximal algorithm in Appendix A.1. 

The Parallel Dykstra-like proximal algorithm can efficiently compute projections onto each convex set in parallel. However, since the DSRIG and SGLIG models require projection onto both $l_2$ and $l_\infty$ norm balls, each iteration becomes computationally expensive. Therefore, while the Parallel Dykstra-like proximal algorithm is time-consuming for the SGLIG model, we instead compute the projection onto the intersection by simple projection onto convex sets (POCS). Given an intersection of sets $K = \bigcap_{i=1}^n k_i$, a sequence of proximal operators could be applied to calculate projection onto $K$ iteratively based on each individual convex set \citep{youla2007image, combettes2011proximal}

\begin{equation}
    prox_K(\beta_{k+1}) = prox_{k_n}(prox_{k_{n-1}}(...prox_{k_1}(\beta_k))).
    \label{eq: sequence of ceonvex sets}
\end{equation}

Computing POCS through a sequence of compositions aligns better with our algorithm and has been shown to be computationally efficient on simulated and real-world datasets. Since our doubly projected proximal algorithm inherits an active set strategy, only a subset of the original predictors is selected. This means the number of convex sets needed for projection is smaller than the number of predictors, which explains why using a sequence of compositions is more efficient than the Parallel Dykstra-like proximal algorithm.

\subsection{Two Stage Projection}
The doubly projected proximal algorithm applies the proximal operator under the $l_1$ and $l_\infty $ constraints. Therefore, we used two steps to calculate the proximal operator under $l_1$ and $l_\infty $ norm balls separately. Let $\mathbf{K}^{i}_{2} := \{ \beta \in \mathbb{R}^p : \lVert \beta_{\mathcal{N}_i}\rVert_{2} \leq \tau_i^* \}$ and $\mathbf{K}^{i}_{\infty} := \{ \beta \in \mathbb{R}^p : \lVert \beta_{\mathcal{N}_i}\rVert_{\infty} \leq \xi_i^* \}$ for $i$ in active group $\hat{\mathbb{G}}^m$. Next, we can find the proximal operator for any $\beta^{(i)} \in \mathbf{K}^{\hat{\mathbb{G}}^m}_{p}$ using 2 stages; (i) a soft thresholding for the within-group individuals $\beta^{(i)}_j \in \beta^{(i)}$ that correspond to the $l_1$ penalty

\begin{equation}
    \left(\arg\min_{\beta \in \mathbf{K}^{i}_{\infty}} \lVert \beta - \beta^{(i)} \rVert\right)_j =
\begin{cases} 
 \xi_j^* & \text{if } \beta^{(i)}_j > \xi_j^*, \\
\beta^{(i)}_j & \text{if } |\beta^{(i)}_j| \leq \xi_j^*, \\
 -\xi_j^* & \text{if } \beta^{(i)}_j < -\xi_j^*.
\end{cases}
\end{equation}
and (ii) a group soft thresholding for the selected groups $\beta^{(i)}$
\begin{equation}
    \arg\min_{\beta \in \mathbf{K}^{i}_{2}} \lVert \beta - \beta^{(i)} \rVert =
\begin{cases} 
 \beta^{(i)}, & \text{if } \|\beta^{(i)}\|_2 \leq \tau_i^*, \\
\tau_i^* \cdot \frac{\beta^{(i)}}{\|\beta^{(i)}\|_2}, & \text{if } \|\beta^{(i)}\|_2 > \tau_i^*.
\end{cases}
\end{equation}
For any neighborhood selected by the doubly projected proximal algorithm,  it is necessary to compute the proximal projection under the constraints of $l_1$ and $l_\infty$ norm balls. Our two-stage projection approach first computes the projections for individual-level sparsity, followed by the group-level proximal projection. In summary, the doubly projected proximal algorithm incorporates the active group strategy with the FISTA algorithm for optimization in the SGLIG model. Up to now, we have provided technical details on performing optimization. In Section 3.6, we derive theoretical properties of the SGLIG estimator.

\subsection{Theoretical Properties of Estimator}
 \citet{rao2013sparse}  and \citet{stephenson2019dsrig} studied the finite sample error bounds and model consistency for their Sparse Overlapping Sets (SOS) LASSO and DSRIG models, which both involve overlapped group LASSO with double sparsity and are similar to the SGLIG model. Building on their analysis, we derive our own error bound with modifications on some assumptions, which can be found in Theorem \ref{th: error bound SGLIG} below. \citet{rao2013sparse} assumed equal weights for group sparsity and within-group sparsity so that the contribution of $l_1$ and $l_2$ to $\mathcal{R}$ is balanced. \citet{stephenson2019dsrig} set upper bounds for the degree of nodes and denote a lower bound and upper  bound for the weight of group sparsity $\tau_i$ such that $1\leq \tau_i\leq \xi$. We develop our own error bound based on these assumptions, but with adjustments to the assumptions on $\tau_i$ and remove the assumption of equal weight between group sparsity and within-group sparsity.

 \citet{negahban2012unified} developed a unified framework for establishing consistency and convergence rates for convex problems. According to this framework, a finite sample error bound exists if the regularizer $\mathcal{R}$ is a decomposable norm and the loss function $\mathcal{L}$ satisfies the restricted
strong convexity (RSC) condition. The following assumptions are considered in this section to satisfy the conditions required for deriving the sample error bound. 

\begin{enumerate}
    \item[(A1)] The decomposition of our regression coefficients into the set of vectors \( V^{(i)} \), \( i = 1, \dots, p \), is an optimal decomposition;
    \item[(A2)] For any node \( i \in J_0 \), \( \mathcal{N}_i \subseteq J_0 \), where $J_0$ represents the collection of true nonzero coefficients;
    \item[(A3)] The true regression parameter vector \( \beta \) is exactly sparse with \( s \) non-zero components that can be decomposed into a set of active vectors \( V^{(i)} \) with at most \( d_{\max} = \max(d_i) \) non-zero elements and at least \( d_{\min} = \min(d_i) \) non-zero elements for \( i = 1, \dots, p \);
    \item[(A4)] $\tau_i$, is lower bounded by \( \sqrt{d_i} \) for \( i = 1, \dots, p \);
    \item[(A5)] The loss function \( L(\beta) \) satisfies the RSC conditions with curvature parameter \( \kappa_L \) (see Appendix A.3, Definition 1);
    \item[(A6)] The design matrix \( X \) is fixed; the observation errors \( \epsilon_k \), \( k = 1, \dots, n \), are additive, independent of \( X \), and \( \epsilon_k \stackrel{i.i.d.}{\sim} \text{Normal}(0, \sigma) \).
    \item[(A7)] As \( n \to \infty \), \( \frac{X^T X}{n} \to M \), where \( M \) is a positive matrix;
    \item[(A8)] The errors \( \varepsilon_1, \varepsilon_2, \dots, \varepsilon_n \) are i.i.d. random variables with mean 0 and finite variance \( \sigma^2 \).
    \label{assp: 18}
\end{enumerate}

\citet{rao2013sparse} mentioned that an optimal decomposition of $\beta$ always exists if the regularizer is convex and coercive. In this context, assumption A1 assumes that our set of $V^{(i)}$ provides the optimal decomposition. A2 was first proposed by \citet{yu2016sparse}, which states that the neighbors of true predictors are also true predictors. It ensures the decomposability of the regularizer. We proposed A3 and A4 to complete the calculation of the finite error bound. A5 is introduced to meet the RSC condition, which is a prerequisite for the unified framework proposed by \citet{negahban2012unified}. A6 is a common assumption for linear regression models. \citet{yu2016sparse} developed assumptions A7 and A8 for the design matrix and errors to study the asymptotic normality for the SRIG model.

We derive a finite error bound and prove asymptotic normality for $\hat{\beta}$, respectively. In particular, the finite error bound is given by,
\begin{theorem}  
Assume (A1)-(A6) for the optimization problem in Equation (\ref{eq: penalized equation}).  Then for \( \lambda^2 \geq \frac{4\sigma^2 \sigma_{\max}^* (\log(p) + d_{\max})}{d_{min} n} \), any optimal solution \( \hat{\beta} \) will satisfy:
\[
\|\hat{\beta} - \beta\|_2^2 \leq \frac{36}{d_{min}} \frac{ \sigma^2 \sigma_{\max}^* \left( \tau_{\max} +  \sqrt{d_{\max}} \right)^2 a(\log(p) + d_{\max})}{n \kappa_L},
\]
with probability at least \( 1 - c_1 \exp(-c_2 \sqrt{n}) \) for some constants \( c_1, c_2 > 0 \), where \( \sigma_i^* \) is the maximum singular value of $X^T_{\mathcal{N}_i} X_{\mathcal{N}_i} $ for $i = 1, \dots, p$, and $\sigma_{\max} = \max_{i = 1, \dots, p} (\sigma_i^*)$.
\label{th: error bound SGLIG}
\end{theorem}
Since the objective function of the SGLIG model involves only one tuning parameter and assumes a different range for $\tau_i$, the finite error bound for the SGLIG estimator is different from the estimator of the DSRIG model. We compare the finite error bounds of the SGLIG and DSRIG estimators under specific conditions. For more details, see Appendix A.

\citet{yu2016sparse} studied the asymptotic normality for the SRIG model with assumptions A7 and A8  for the design matrix and errors.  
\
Using the same assumptions, we also show the existence of asymptotic normality of the SGLIG model.
\begin{theorem} \label{thm:asymptotic_normality} Given assumptions A2, A4, A7, and A8, suppose the scale parameter \( \lambda^* \) and weights \( \tau_i \)'s are chosen such that
\[
\sqrt{n} \lambda^* \to 0 \quad \text{and} \quad n^{(\gamma + 1)/2} \lambda^* \to \infty \quad \text{for some} \ \gamma > 0.
\]
Then, with dimension \( p \) fixed, as \( n \to \infty \), we have
\[
\sqrt{n} (\hat{\beta}_{J_0} - \beta_{J_0}) \xrightarrow{d} N(0, \sigma^2 M_{J_0, J_0}^{-1}),
\]
and
\[
\hat{\beta}_{J_0^c} \xrightarrow{p} 0,
\]
where \( M_{J_0, J_0} \) is the submatrix of \( M \) consisting of the entries with row and column indices in \( J_0 \).

\end{theorem}
Theorem \ref{thm:asymptotic_normality} indicates that the SGLIG model achieves the same estimation consistency as the SRIG model when the design matrix has a fixed dimension $p$. In other words, the SGLIG model retains asymptotic normality with the additional $l_1$ penalty. See Appendix A.4 for a full proof.

\section{Simulation Study}
\label{sec: simulation}
In this section, we describe the design of multiple simulation studies to evaluate the performance of our SGLIG model compared to that of the SRIG and DSRIG models. Additionally, we measure the computational cost of the three models, with results showing that the SGLIG model is more efficient than the DSRIG model while maintaining competitive prediction accuracy.

\subsection{Simulation Design}
The simulation studies aim to evaluate the performance of the SGLIG model on datasets generated with varying graphical structures, enabling us to assess the robustness of the SGLIG model across different types of datasets. All the codes were running on an ASUS PC with a 2.6GHz Intel six-core CPU, using a multiprocessing package in Python with pool(12). The computation time illustrated in the tables is the average running time over all simulated datasets. 

For the SRIG and SGLIG models, we considered 50 possible values for tuning parameters $\lambda$ and $\alpha$, while one additional tuning parameter $\xi$ was used for the DSRIG model. Thus, the tuning parameters for the DSRIG model are ($\lambda$, $\xi$), and there were $50\times50$ combinations of tuning parameters for the DSRIG model. We designed simulation datasets with five graphical structures. As proposed in Section 2, $\Omega$ represents the conditional dependencies among the predictors, thereby demonstrating the structural information. Therefore, $\Omega$ was required before generating the simulated data.  To ensure the numerical stability of the precision matrix, we used $\Omega = B + \delta I$, as proposed by \citet{yu2016sparse}, where $B$ represents the off-diagonal values of the precision matrix that defines the graph structure, $I$ is the $p$-dimensional identity matrix, and $\delta$ is the value selected such that the conditional number of $\Omega$ equals $p$. The diagonal values of matrix $B$ are always zero, while its off-diagonal values are determined by the graph structure. Subsequently, $\Omega$ was standardized to have unit diagonals. 

For each graph type, we generated $\Omega$ to create 10 true parent graphs with $p = 100$. From each parent graph, 10 datasets with dimensions $560\times100$ were generated, resulting in a total of 100 datasets per graph type. Two scenarios of datasets were generated, varying by sample size and partition proportions. Each dataset was divided into training, validation, and testing sets. In Scenario 1, the sample size was $n$ = 480, with a split of 40/40/400. In Scenario 2, the sample size was $n$ = 560, with a split of 80/80/400. Both $\Omega$ and datasets were generated using functions from the NumPy package \citep{harris2020array}. The response variable was generated using (\ref{eq:linear_model}) where $\beta = \Omega \Sigma_{xy}$ and individual errors were normally distributed with a mean equal to 0 and a standard deviation $\sigma = 5$. To define $\Sigma_{xy} = (c_1, c_2, \dots, c_p)$, we randomly selected four numbers from the range [1,100]  representing four randomly selected nodes, and assigned $c_i = 4$ for these selected nodes.  For all other unselected nodes, we set $c_i = 0$. 

In the SRIG model, the tuning parameter $\lambda$ controls the overall shrinkage of the group LASSO. We first identified $\lambda_{max}$ such that it results in a null $\beta$. The minimum value of $\lambda$ was set to $0.01\times\lambda_{max}$.  The optimal $\lambda$ was then selected from a grid of 50 equally spaced values on a logarithmic scale between $\lambda_{min}$ and $\lambda_{max}$. The weight $\tau_i$ was calculated using $\frac{1}{cov(X_i, Y)}$, which is inversely related to the correlation between $X_i$ and the response variable $Y$. For the DSRIG model, the selection of $\lambda$ followed the same procedure as in the SRIG model. Also, the parameter $\xi$ was selected from a grid of 50 equally spaced values between 0 and a sufficiently large value $\xi_{max}$ to ensure it will not be selected as the optimal value. We used $\xi_{max} = 5$. In the SGLIG model, $\frac{\lambda_{max}}{5}$ was used to control the overall shrinkage between and within group sparsity. The tuning parameter $\alpha$ was selected from a grid of equally spaced values on a logarithmic scale between 0.01 and 1. Table \ref{tab:1} illustrates the performance of the graphical models across the five graph types. We now introduce the five graph types explored in this study.

\subsection*{Scenario 1: Two Class}
A two-class graph structure indicates that the nodes in the graph are divided into two classes. In class one, the nodes are highly ‘active’, meaning they have a high probability of connecting with other nodes, both within the same class and with nodes in the other class. In contrast, the nodes in the other class are ‘inactive’, exhibiting a low probability of connecting with any other nodes. We assigned the first 20 predictors to class one and the remaining 80 predictors to class two. Therefore, let $X = (X_1,X_2,\dots, X_p)^T$, where $p$ = 100. We generated $X\sim \mathcal{MVN}(0, \Omega^{-1})$ where $\Omega = B + \delta I$. Denoting $B_{ij}$ as the off-diagonal elements in matrix $B$, we set $B_{ij}$ equal to 0.5 with a probability of 0.1, or equal to 0 with a probability of 0.9 if $i\in$ class one, $j\in$ class two, or both $i\in$ class one and $j\in$ class one. Otherwise, $B_{ij}$ equals 0.5 with a probability of 0.05 or 0 with a probability of 0.95.

\subsection*{Scenario 2: Bipartite}
Similar to the two-class structure, the bipartite graph divides the nodes into two distinct sets, $U$ and $V$. However, in a bipartite graph, every edge must connect one vertex from set $U$ to a vertex in set $V$. Edges connecting two nodes within the same set are not permitted. In Scenario 2, the first 20 predictors were assigned to $U$, while the remaining 80 predictors were assigned to $V$. Again, for $X$ with 100 predictors, let $X\sim \mathcal{MVN}(0, \Omega^{-1})$ and $\Omega = B + \delta I$. We set $B_{uv} = B_{vu}$ equal to 0.5 with a probability of 0.1 or 0 with a probability of 0.9. 

\subsection*{Scenario 3: Random}

The random graph structure assumes no inherent organization among the graph edges, meaning that connections exist purely at random. Each node has a constant probability of connecting with any other node. Using 100 predictors, we sampled $X\sim \mathcal{MVN}(0,\Omega^{-1})$, where $\Omega = B + \delta I$, where $B_{ij}$ denotes the off-diagonal elements of matrix $B$. We set $B_{ij} = B_{ji }= 0.5$ with a probability of 0.05 or 0 with a probability of 0.95.

\subsection*{Scenario 4: Blockwise}

A blockwise graph divides the graph into several ``blocks", where each block represents a subgraph of the primal graph. For $X$ with 100 predictors, where $X\sim \mathcal{MVN}(0,\Omega^{-1})$ and $\Omega = B + \delta I$, we defined three distinct blocks, each containing 10 predictors, while the remaining 70 predictors are disconnected. Subgraphs are independent, with no edge connecting nodes in different blocks. Within each block, nodes are randomly connected with a probability of 0.5. Consequently, for $1\leq i,j \leq 30$, the off-diagonal values within each block satisfy $B_{ij} = 0.5$ with a probability of 0.5, or $B_{ij} = 0$ with a probability of 0.5.

In summary, the blockwise graph consists of three independent random subgraphs, each containing 10 predictors, while the other predictors remain disconnected. For simplicity, we applied the same connection probability for each block in our simulation; however, these probabilities could vary in real-world data.

\subsection*{Scenario 5: Band}

A band graph typically represents a line graph with a specified length or distance.  In our band graph, we assumed all the predictors are arranged linearly in order from 1 to $p$. For $X$ with 100 predictors, where $X\sim \mathcal{MVN}(0,\Omega^{-1})$ and $\Omega = B + \delta I$, we assumed a linear band structure among the 100 predictors. Specifically, we set $B_{ii} = 1.333$, and $B_{ij} = -0.667$ if $\lvert i-j \rvert = 1$, with all other $B_{ij}$ values equal to 0. These specific values were chosen such that $\Sigma_{ij} = 0.5^{\lvert i-j \rvert}$  as described by \citet{yu2016sparse}. Consequently, our graph represents a 100-length linear graph ordered from 1 to 100.

\subsubsection{Estimated Graph}

When analyzing real-world datasets, the true graph structure is often unknown. To make our simulations more reflective of real-world scenarios, we assume the graph structure is not known in advance and find an estimated graph for prediction. Concretely, we use the graph-estimation technique proposed by \citet{meinshausen2006high}(MB method), which iteratively applies the LASSO regression, regarding one variable as the response variable and the remaining variables as predictors. The predictors with non-zero coefficients are then identified as connected to the response variable, thus generating an estimated graph. To ensure sparsity in the estimated graph, which is common in most real-world graphs, we set the regularization parameter of the MB method to 0.5.

\subsection{Simulation Result}
We fit the SRIG, SGLIG, and DSRIG models to two scenarios of datasets with different splits of train data, validation, and test data. In this section, we first compare the performance of the graphical models and then evaluate their performance with various numbers of nodes in the graph to determine whether the model has stable performance with scalable graph information.

\subsubsection{Performance of Graphical Models}
From Table \ref{tab:1}, we observe that the DSRIG model outperforms both the SRIG and SGLIG models in terms of $l_2$ distance and mean squared error (MSE) for both scenarios. However, the improved performance of the DSRIG model comes with higher computational costs compared to the SRIG and SGLIG models. 
\renewcommand{\thetable}{\arabic{table}}

\begin{longtable}{p{2.5cm} c p{1cm} S[table-format=2.2] c c p{1cm} S[table-format=2.2]}
\caption{Performance of SRIG, SGLIG, and DSRIG in terms of three key metrics: $l_2$ distance ($\lVert \hat{\beta} - \beta \rVert_2$), mean square error, and computation time (seconds) when applied to two scenarios of datasets, each with 100 predictors.}
\label{tab:1} \\
\toprule
\multicolumn{1}{c}{} & \multicolumn{3}{c}{40/40/400} & \multicolumn{1}{c}{} & \multicolumn{3}{c}{80/80/400} \\
\cmidrule(r){2-4} \cmidrule(r){6-8}
Model & $l_2$ distance & MSE & {Time} & & $l_2$ distance & MSE & {Time} \\
\midrule
\endfirsthead

\toprule
Model & $l_2$ distance & MSE & {Time} & & $l_2$ distance & MSE & {Time} \\
\midrule
\endhead

\bottomrule
\endfoot

\textbf{Two-Class} & & & & & & & \\
SRIG  & 8.86 & 9.55 & 3.11 & & 7.60 & 8.64 & 2.98 \\
SGLIG & 8.15 & 9.23 & 3.43 & & 5.36 & 7.14 & 5.94 \\
DSRIG & 7.63 & 8.86 & 130.01 & & 4.97 & 6.92 & 92.24 \\
\addlinespace

\textbf{Bipartite} & & & & & & & \\
SRIG  & 8.05 & 9.24 & 3.65 & & 6.89 & 8.38 & 2.95 \\
SGLIG & 6.87 & 8.40 & 3.79 & & 4.46 & 6.67 & 5.14 \\
DSRIG & 6.18 & 7.88 & 107.10 & & 4.01 & 6.36 & 90.34 \\
\addlinespace

\textbf{Random} & & & & & & & \\
SRIG  & 8.76 & 9.28 & 3.81 & & 7.21 & 8.32 & 3.01 \\
SGLIG & 8.09 & 9.06 & 4.39 & & 5.55 & 7.24 & 5.60 \\
DSRIG & 7.67 & 8.74 & 111.79 & & 5.09 & 7.01 & 89.24 \\
\addlinespace

\textbf{Blockwise} & & & & & & & \\
SRIG  & 8.18 & 9.23 & 2.84 & & 6.55 & 7.95 & 1.43 \\
SGLIG & 6.76 & 8.29 & 4.80 & & 3.91 & 6.33 & 2.64 \\
DSRIG & 6.16 & 7.84 & 102.84 & & 3.55 & 6.16 & 60.22 \\
\addlinespace

\textbf{Band} & & & & & & & \\
SRIG  & 9.33 & 9.40 & 3.83 & & 8.32 & 8.88 & 3.15 \\
SGLIG & 9.23 & 9.48 & 3.27 & & 7.38 & 8.50 & 6.10 \\
DSRIG & 8.79 & 9.38 & 98.31 & & 5.30 & 7.37 & 94.19 \\
\end{longtable}

In contrast, there is a balance between MSE and computational efficiency for the SGLIG model. In Scenario 1, the SGLIG model has a computational cost comparable to the SRIG model but delivers outstanding performance(MSE) across all graph types, except for the band graph. This exception may arise because our band graph is designed to be a connected line graph where each node is linked to at most two others, so within-group sparsity is not as desirable for this particular graph structure. For two-class, bipartite, random, and blockwise graphs, the SGLIG model performs competitively with the DSRIG model, while having significantly lower computational costs. In summary, the SGLIG model operates as a compromise between the SRIG and DSRIG models. It includes double sparsity with efficiency by using only one tuning parameter, making computational costs close to the SRIG model while maintaining performance competitive to the DSRIG model in all scenarios.

To evaluate the performance of the SGLIG model in higher dimensions, we extended the analysis to compare model performance in higher-dimensional settings. Specifically, we fit the SRIG, SGLIG, and DSRIG models to datasets with 100, 200, 400, and 800 predictors. Table \ref{tab:2} demonstrates the performance of the SRIG and SGLIG models for a two-class plot with 480 observations, split as 40/40/400 for training, validation, and testing, respectively.

\begin{longtable}{p{2cm} c S[table-format=5.1] S[table-format=2.2] S[table-format=2.2] S[table-format=4.2]}
\caption{Performance of SRIG, SGLIG, DSRIG in terms of three key metrics: $l_2$ distance ($\lVert \hat{\beta} - \beta \rVert_2$), mean square error, and computation time (seconds) when applied to high-dimensional two-class graph.}
\label{tab:2} \\
\toprule
Model & Predictors & {Edge} & {$l_2$ distance} & {MSE} & {Time} \\
\midrule
\endfirsthead
\toprule
Model & Predictors & {Edge} & {$l_2$ distance} & {MSE} & {Time} \\
\midrule
\endhead
\bottomrule
\endfoot
SRIG  & 100  &  333.7  & 7.60 & 8.64 & 2.98 \\
SGLIG & 100  &  333.7  & 5.36 & 7.14 & 5.94 \\
DSRIG & 100  &  333.7  & 4.97 & 6.92 & 92.24 \\
\addlinespace
SRIG  & 200  & 1339.8  & 8.76 & 9.25 & 6.31 \\
SGLIG & 200  & 1339.8  & 7.49 & 8.57 & 10.42 \\
DSRIG & 200  & 1339.8  & 6.96 & 8.08 & 351.12 \\
\addlinespace
SRIG  & 400  & 5419.9  & 8.92 & 9.38 & 16.15 \\
SGLIG & 400  & 5419.9  & 8.48 & 9.29 & 25.07 \\
DSRIG & 400  & 5419.9  & 7.83 & 8.73 & 810.59 \\
\addlinespace
SRIG  & 800  & 21687.0 & 8.94 & 9.53 & 40.63 \\
SGLIG & 800  & 21687.0 & 8.75 & 9.46 & 59.60 \\
DSRIG & 800  & 21687.0 & 8.21 & 9.04 & 2932.11 \\

\end{longtable}
In Table \ref{tab:2}, the SGLIG model consistently outperforms SRIG in both $l_2$ distance and MSE across all scenarios. However, the SGLIG model requires more computational time than the SRIG model since it incorporates double sparsity, which naturally increases computational complexity compared to the SRIG model with single sparsity. Thus, the increased computational time is understandable and tolerable, given the improvement in the model performance.

\section{Real Data Analysis}
\label{sec: real data}
We fit the SGLIG model to two real-world datasets and compute the optimal $\hat{\beta}$ using the doubly projected proximal algorithm. For both datasets, we standardized the columns of the design matrix and the chosen response variable, ensuring that each variable has a mean of 0 and a standard deviation of 1. Additionally, both datasets were assumed to follow multivariate normal distributions, maintaining the interpretability of the Gaussian graphical model, where the precision matrix reflects the conditional dependence among predictors. To evaluate the performance of the SGLIG model with a finite sample size, we split each dataset into roughly 10 segments. One segment was used as the testing dataset to measure model performance. From the remaining 9 segments, we selected 8 as the training data and used the remaining 1 segment as the validation set to determine the optimal value of the tuning parameter. In total, there are 90 permutations of the dataset. A consensus estimated graph was generated for both datasets by using the edges that appear in all 90 estimated graphs from the permutations. The performance of our method was evaluated using computation time and MSE, based on the average results across all 90 permutations. For brevity, we show results for the blood-brain barrier data here.  However, in Appendix B, we analyze data on brain volumes derived from MRI images and their associations with a measure of mental cognition.

\subsection{Blood Brain Barrier Data}
The first real-world dataset is the blood-brain barrier data from the caret package \citep{kuhn2015caret} for R \citep{r2013r}. The blood-brain barrier data consists of the response variable, corresponding to the log of the ratio of the concentration of a compound in the brain to the concentration in the blood, along with 134 chemical descriptors (e.g., mass, surface value of compound, etc.) for $n=208$ chemical compounds \citep{mente2005recursive}. As mentioned earlier, we assumed all the predictors in the blood-brain barrier dataset were normally distributed. However, the dataset includes binary, ordinal, and continuous variables. Two binary predictors with minimal variability were excluded, leaving $p=132$ chemical descriptors in the analysis. Table \ref{tab: blood_brain} compares the performance of SRIG, SGLIG, and DSRIG models on the blood-brain barrier dataset.

\begin{longtable}{p{2cm} c S[table-format=1.1] c p{1cm} S[table-format=2.2]}
\caption{Performance of SRIG, SGLIG, DSRIG on the blood-brain barrier data in terms of two key metrics:  mean square error, and computation time (seconds).} \\
\toprule
Model & Predictors & {Edge} & non-zero & {MSE} & {Time} \\
\midrule
\endfirsthead
\toprule
Model & Predictors & {Edge} & non-zero & {MSE} & {Time} \\
\midrule
\endhead
\bottomrule
\endfoot
\endlastfoot
SRIG  & 132 & 90.89 & 10.76 & 0.66 & 3.76 \\
SGLIG & 132 & 90.89 & 29.14 & 0.59 & 6.09 \\
DSRIG & 132 & 90.89 & 33.81 & 0.60 & 80.01 \\
\bottomrule
\label{tab: blood_brain}
\end{longtable}
From the Table \ref{tab: blood_brain}, we observe that the SGLIG model delivers the best prediction in terms of MSE, but the DSRIG model also performs well and is competitive with SGLIG. The SRIG model achieves the lowest computation time but also the lowest MSE. Notably, the computation time of the SGLIG model is close to that of the SRIG model, while its MSE is the lowest among all the models. We further aggregate the estimated graphs across the 90 permutations to construct a consensus graph, shown in Figure \ref{fig: bb estimated graph}. Edges are included based on their frequency of occurrence; specifically, an edge is retained if it appears in more than 70 of the 90 permutations, yielding a sufficiently explicit consensus structure.
\begin{figure}[h!]
    \centering
    \includegraphics[width=0.8\textwidth]{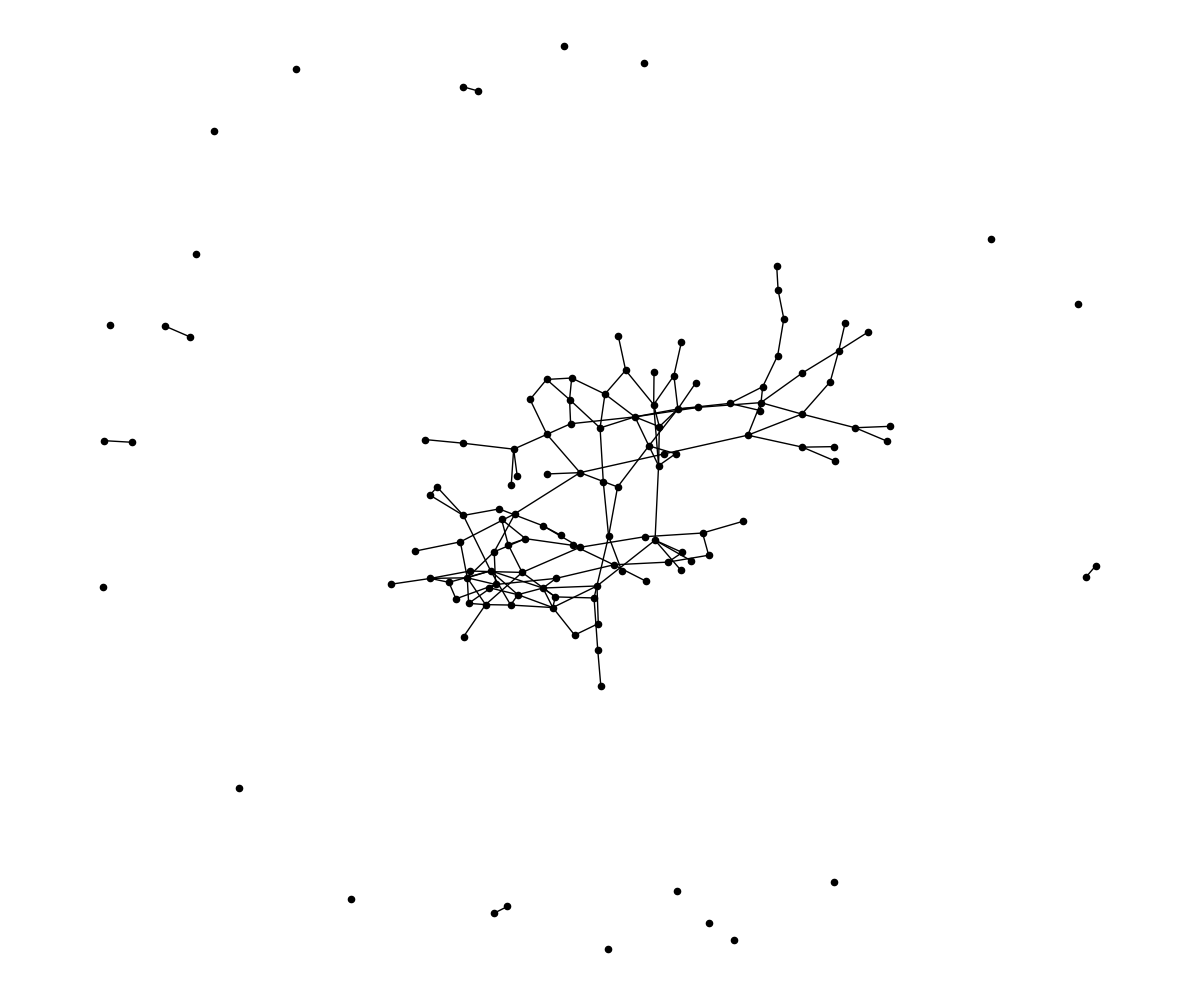} 
    \caption{Undirected consensus graph for the blood brain barrier data}
    \label{fig: bb estimated graph}
\end{figure}
\\
Figure \ref{fig: bb graph} compares the MSE differences between SRIG, SGLIG, and DSRIG across all 90 permutations of the blood-brain barrier dataset. The differences are calculated as SRIG-SGLIG and DSRIG-SGLIG. Points above zero indicate that the SGLIG model outperforms the others. Notably, in the SRIG-SGLIG plot, most of the points are above zero, suggesting that the SGLIG model outperforms the SRIG model in the majority of permutations of blood-brain barrier datasets. Additionally, in the DSRIG-SGLIG plot, we observe that the SGLIG and DSRIG models perform competitively. In conclusion, the SGLIG model accomplishes a trade-off between SRIG and DSRIG, spending much less time than the DSRIG model while retaining competitive prediction accuracy for the blood-brain barrier dataset.
\begin{figure}[h!]
    \centering
    \includegraphics[width=1.1\textwidth]{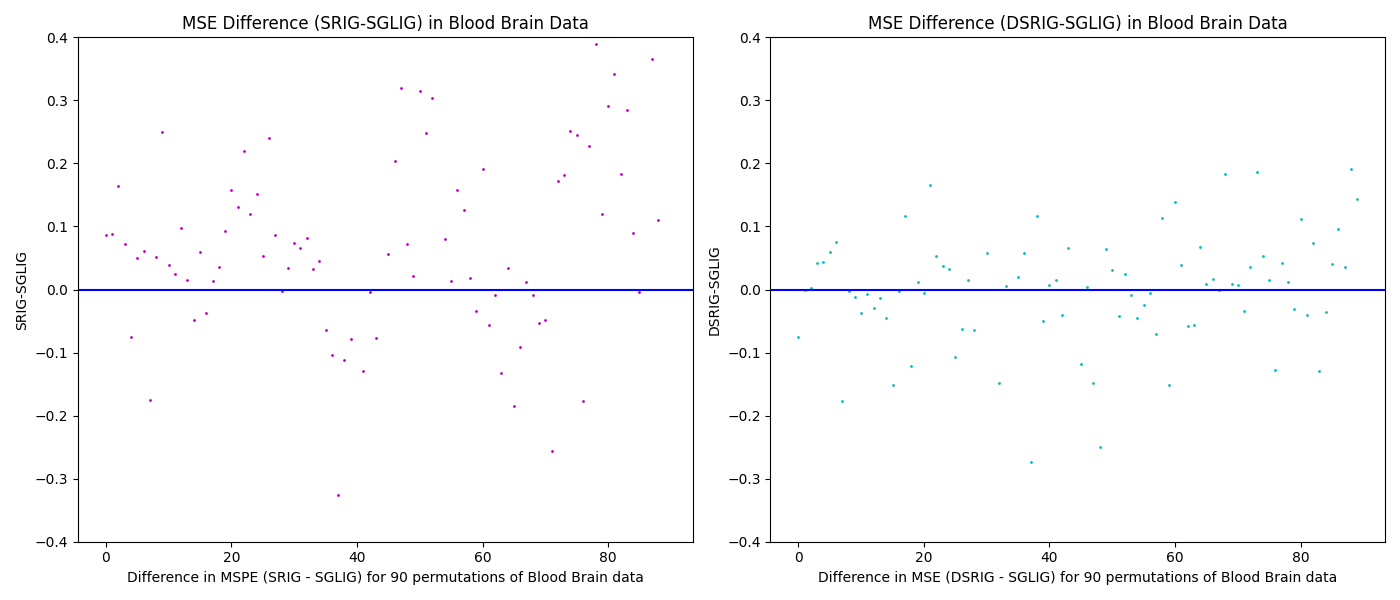} 
    \caption{MSE for 90 permutations of blood-brain data}
    \label{fig: bb graph}
\end{figure}
Table \ref{table: consisbb} lists the variables that were consistently selected in over 70 out of 90 permutations for the blood-brain barrier data. The SGLIG and DSRIG models exhibit similar patterns of variable selection compared to the SRIG model, as both SGLIG and DSRIG models incorporate double sparsity. Unlike SRIG, the doubly sparse methods selected predictors related to particle charge.  However, the variables identified by SRIG were mainly measures related to compound surface area, while SGLIG and DSRIG were more selective in which surface area properties were important.
\begin{longtable}{llccc}
\caption{Selected variables that were selected by over 70 out of 90 permutations} \\
\hline
           &           & \multicolumn{3}{c}{Selected Variables} \\ \cline{3-5} 
Variable                    &  Descriptor                   & SRIG & SGLIG & DSRIG \\ \hline
\endfirsthead
\hline
           &           & \multicolumn{3}{c}{Selected Variables} \\ \cline{3-5} 
Variable                    &  Descriptor                   & SRIG & SGLIG & DSRIG \\ \hline

\endhead
\bottomrule
\endfoot
\endlastfoot
tpsa               & topological polar surface area  & \checkmark    & \checkmark     & -     \\ 
peoe\_vsa.1.1      & partial electrostatic charges\_vsa.1.1 & \checkmark    & -     & -     \\ 
vsa\_other         & van der Waals surface area & \checkmark    & \checkmark     & \checkmark     \\ 
tpsa.1             & topological polar surface area\_1  & \checkmark    & -     & -     \\ 
mlogp              & Moriguchi octanol-water partition coefficient  & \checkmark    & -     & -     \\ 
clogp              & partition coefficient calculator for decades  & \checkmark    & \checkmark     & \checkmark     \\ 
nocount            & na & \checkmark    & -     & -     \\ 
prx                & Moriguchi parameters  & \checkmark    & \checkmark     & \checkmark     \\ 
polar\_area        & polar surface area & \checkmark    & -     & -     \\ 
peoe\_vsa.3        & partial electrostatic charges\_vsa.3 & -    & \checkmark     & \checkmark     \\ 
peoe\_vsa.4.1      & partial electrostatic charges\_vsa.4.1 & -    & \checkmark     & \checkmark     \\ 
slogp\_vsa1        & logP\_vsa1 & -    & \checkmark     & \checkmark     \\ 
fpsa3              & fractional charged partial positive surface area 3 & -    & \checkmark     & \checkmark     \\ 
vsa\_acid          & vsa\_acid  & -    & -     & \checkmark     \\ 
\bottomrule
\label{table: consisbb}
\end{longtable}

\section{Conclusion}
\label{sec: conclusion}
In this paper, we proposed the novel Sparse overlapping Group LASSO Incorporating Graphical structure (SGLIG). We showed that this model is computationally efficient while retaining competitive performance compared to the DSRIG model. Instead of using traditional predictor duplication methods for computing the SGLIG model with overlapping groups, we developed the doubly projected proximal algorithm, which iteratively projects the estimator onto convex sets. While the predictor duplication method is convenient for datasets with sparse graph structures, it becomes computationally expensive for high-dimensional datasets or dense graph structures. Our proposed doubly projected proximal algorithm calculates the projection iteratively rather than duplicating the variables, requiring fewer computational resources when processing high-dimensional datasets with dense graph structures. Furthermore, we derived the finite sample error bound for the estimator calculated by the SGLIG model and compared it with the DSRIG model. We showed that under certain conditions, the SGLIG model can have a smaller finite error bound than the DSRIG model. The estimator computed by the SGLIG model was also proven to exhibit asymptotic normality under pre-specified assumptions. In our simulation study, we fit the SRIG, SGLIG, and DSRIG models to datasets generated with five different graph structures. Our results indicated that the SGLIG model retains competitive performance compared to the DSRIG model, while its computational time is close to that of the SRIG model. Finally, the SGLIG model was applied to real-world datasets and shown to outperform the DSRIG model while significantly reducing computation time. 

\section{Disclosure statement}\label{disclosure-statement}
The authors declare that they have no conflicts of interest.
\section{Acknowledgements}
This work was funded by the Natural Sciences and Engineering Research
Council of Canada. Data collection and sharing for the Alzheimer's Disease Neuroimaging Initiative (ADNI) is funded by the National
Institute on Aging (National Institutes of Health Grant U19 AG024904). The grantee organization is the Northern
California Institute for Research and Education. In the past, ADNI has also received funding from the National
Institute of Biomedical Imaging and Bioengineering, the Canadian Institutes of Health Research, and private
sector contributions through the Foundation for the National Institutes of Health (FNIH) including generous
contributions from the following: AbbVie, Alzheimer’s Association; Alzheimer’s Drug Discovery Foundation;
Araclon Biotech; BioClinica, Inc.; Biogen; Bristol-Myers Squibb Company; CereSpir, Inc.; Cogstate; Eisai Inc.;
Elan Pharmaceuticals, Inc.; Eli Lilly and Company; EuroImmun; F. Hoffmann-La Roche Ltd and its affiliated
company Genentech, Inc.; Fujirebio; GE Healthcare; IXICO Ltd.; Janssen Alzheimer Immunotherapy Research \&
Development, LLC.; Johnson \& Johnson Pharmaceutical Research \&Development LLC.; Lumosity; Lundbeck;
Merck \& Co., Inc.; Meso Scale Diagnostics, LLC.; NeuroRx Research; Neurotrack Technologies; Novartis
Pharmaceuticals Corporation; Pfizer Inc.; Piramal Imaging; Servier; Takeda Pharmaceutical Company; and
Transition Therapeutics.


\bibliography{bibliography}
\newtheorem{lemma}{Lemma}
\newtheorem{definition}{Definition}
\appendix
\section{Appendix A: Detailed Derivations}

\subsection{Parallel Dykstra-like Proximal Algorithm\citep{combettes2011proximal}}
\label{sec: Dykstra}
Suppose there is a single closed and convex set $A$; the projection can be denoted as 
\begin{equation}
    \forall x \in \mathbb{R}^p, \; \pi_A(\beta^{(i)}) = \arg\min_{\beta \in A} \|\beta - \beta^{(i)}\|.
\end{equation}
Given an intersection of sets $K = \bigcap_{i=1}^n k_i$ and $z_{i,0} = \beta^{0}$ $\forall i \in \{1, \dots, k\}$, the Parallel Dykstra-like proximal algorithm computes the projection onto $K$ as follows

\begin{align*}
\beta^{(n+1)} &= \sum_{i=1}^{n} \pi_{k_i} \left( z_{i,n} \right), \\
z_{i,n+1} &= \beta^{(n+1)} + z_{i,n} - \pi_{k_i} \left( z_{i,n} \right).
\end{align*}

\subsection{Doubly Projected Proximal Algorithm}
\begin{lemma}
For any $\lambda > 0$ and $p \geq 1$, the proximity operator of $\Omega^{\mathbb{G}}_{p}$, where $\Omega^{\mathbb{G}}_{p} = \\\min_{\sum_{i=1}^{p} V^{(i)} = \beta, supp(V^{(i)}) \subseteq \mathcal{N}_i} \sum_{i=1}^{p} \tau_{i}\lVert V^{(i)} \rVert_2 + \xi\lVert V^{(i)} \rVert_1$, is given by 
\begin{center}{$prox_{\Omega^{\mathbb{G}}_{p}}(\beta_{\mathcal{N}_i}) = \mathbb{I} - \arg\min_{\beta \in \mathbf{K}^{\mathbb{G}}_{p}} \lVert \beta - \beta_{\mathcal{N}_i} \rVert,$} \\
\end{center}
where $\mathbf{K}^{\mathbb{G}}_{p} = \{ \beta \in \mathbb{R}^p, \lVert \beta \rVert_{\mathcal{N}_i,2}\leq {\tau^{*}_{i}} \And  \lVert \beta \rVert_{\mathcal{N}_i,\infty}\leq {\xi^{*}}, for \,i = 1,...,p\}$, where $\tau^{*}_{i} = 2\tau_{i}, \xi^{*} = 2\xi$.
\label{lem: 1}
\end{lemma}

\begin{proof}
First, we compute the Fenchel conjugate of $\Omega^{\mathbb{G}}_{p}$. By definition, we have\\
\begin{equation} 
\begin{split}
(\Omega^{\mathbb{G}}_{p})^{*} & = \sup_{\beta \in \mathbb{R}^p}\left(\langle \beta, \mu \rangle -\min \sum_{i=1}^{p} (\tau_{i}\|V^{(i)}\|_2 + \xi \|V^{(i)}\|_1) \right) \\
 & = \sup_{\beta \in \mathbb{R}^p}\left(\sup_{\sum_{i=1}^{p} V^{(i)} = \beta}\langle \beta, \mu \rangle - \sum_{i=1}^{p} (\tau_{i}\|V^{(i)}\|_2 + \xi \|V^{(i)}\|_1) \right)\\
 & =  \sup_{\sum_{i=1}^{p} V^{(i)} = \beta}\left(\langle \sum_{i=1}^{p} V^{(i)}, \mu \rangle - \sum_{i=1}^{p} (\tau_{i}\|V^{(i)}\|_2 + \sum_{i=1}^{p} \xi \|V^{(i)}\|_1) \right)\\
 & = \sum_{i=1}^{p}\sup_{ \sum_{i=1}^{p}{V^{(i)} = \beta}}\left(\langle \ V^{(i)}, \mu^i \rangle -  \tau_{i}\|V^{(i)}\|_2 -  \xi \|V^{(i)}\|_1 \right)\\
 & = \sum_{i=1}^{p}\sup_{ \sum_{i=1}^{p}V^{(i)} = \beta}(\frac{1}{2}\langle  V^{(i)}, \mu^i \rangle - \tau_{i}\|V^{(i)}\|_2) + \sum_{i=1}^{p}\sup_{ \sum_{i=1}^{p}V^{(i)} = \beta}(\frac{1}{2}\langle  V^{(i)}, \mu^i \rangle - \xi \|V^{(i)}\|_1 )\\
  & = \sum_{i=1}^{p}\tau_i\sup_{ \sum_{i=1}^{p}{V^{(i)}} = \beta}(\langle  V^{(i)}, \frac{\mu^i}{2\tau_i}\rangle - \|V^{(i)}\|_2) + \sum_{i=1}^{p}\xi\sup_{ \sum_{i=1}^{p}V^{(i)} = \beta}(\langle  V^{(i)}, \frac{\mu^i}{2\xi} \rangle -  \|V^{(i)}\|_1 )\\
  & = \sum_{i=1}^{p} \tau_i \iota_2(\frac{\mu^i}{2\tau_i}) + \sum_{i=1}^{p} \xi \iota_{\infty}(\frac{\mu^i}{2\xi})\\
  & = \begin{cases}
          0 \quad &\text{if} \, \quad \lVert \mu^i \rVert_{2}\leq {2\tau_{i}} \And  \lVert \mu^i \rVert_{\infty}\leq {2\xi}\ for\quad i = 1,...,p\\
          +\infty \quad &\text{Otherwise}  \\
     \end{cases}.
\end{split}
\end{equation}
Now set $\mathbf{K}^{\mathbb{G}}_{p} = \{ \beta \in \mathbb{R}^p, \lVert \beta \rVert_{\mathcal{N}_i,2}\leq {\tau^{*}_{i}} \And  \lVert \beta \rVert_{\mathcal{N}_i,\infty}\leq {\xi^{*}}\, for\, i = 1,...p \}$ where $\tau^{*}_{i} = 2\tau_{i}, \xi^{*} = 2\xi$.\\
Then we have $(\Omega^{\mathbb{G}}_{p})^{*}(\mu) = \iota_{\mathbf{K}^{\mathbb{G}}_{p}}(\mu)$, where $\iota$ is the unitary ball in $\mathbb{R}^p$, and it is well known that
\begin{center}{
$prox_{\lambda\iota_{\mathbf{K}^{\mathbb{G}}_{p}}}(x) = \arg\min_{y \in \mathbf{K}^{\mathbb{G}}_{p}} \lVert y - x \rVert$}.
\end{center}
Now we apply the Moreau decomposition\\
\begin{equation} 
\begin{split}
prox_{\lambda\Omega^{\mathbb{G}}_{p}}(x) & = x - \lambda prox_{\frac{\Omega^{\mathbb{G}^{*}}_{p}}{\lambda}}(x)\\
& = x - \lambda\arg\min_{y \in \mathbf{K}^{\mathbb{G}}_{p}} \lVert y - \frac{x}{\lambda} \rVert\\
& = x - \arg\min_{y \in \lambda\mathbf{K}^{\mathbb{G}}_{p}} \lVert y - x \rVert.\\
\end{split}
\end{equation}
\end{proof}

\begin{lemma}

Given $x \in \mathbb{R}^d$, it holds 
\begin{center}{$\arg\min_{y \in \mathbf{K}^{\mathbb{G}}_{p}} \lVert y - x \rVert =  \arg\min_{y \in \mathbf{K}^{\hat{\mathbb{G}}}_{p}} \lVert y - x \rVert$}, 
\end{center}

where $\hat{\mathbb{G}} := \{i \in \mathbb{G}, \|x\|_{\mathcal{N}_i,2} > \tau^{*}  \,or\,  \|x\|_{\mathcal{N}_i,\infty} > \xi^{*}\}.$
\label{lem: 2}
\end{lemma}
\begin{proof}

Given a group of indices $\mathcal{N}_i$, then for any subset $\mathbb{S} \in \mathbb{G}$, we have

$\mathbb{K}^{\mathbb{S}} = \{ \beta \in \mathbb{R}^p, \lVert \beta \rVert_{\mathcal{N}_i,2}\leq {\tau^{*}} \And  \lVert \beta \rVert_{\mathcal{N}_i,\infty}\leq {\xi^{*}}\, for\, i \in \mathbb{S} \}.$
\\
First, we need to prove the projection onto $\mathbb{K}^{\mathbb{S}}$ is non-expansive coordinate-wise with respect to zero. In other words, we need to show that $\arg\min_{y \in \mathbf{K}^{{\mathbb{S}}}} \| y - x \|_i \leq \|x_i\|$ for all $i$ = 1,..., $p$. This can be proved by contradiction. Assume there exists an index $\hat{j}$ such that $\arg\min_{y \in \mathbf{K}^{{\mathbb{S}}}} \| y - x \|_{\hat{j}} \leq \|x_{\hat{j}}\|$. We now assume that there exists a vector $\hat{x}_j$ such that \\
\begin{center}
  $ \hat{x}_j = \begin{cases}
          \arg\min_{y \in \mathbf{K}^{{\mathbb{S}}}} \| y - x \|_{{j}} \quad if\, j\neq \hat{j}\\
          x_{\hat{j}} \quad \quad\quad\quad\quad\quad\quad\quad\quad\, Otherwise \\
     \end{cases}.$
\end{center}
Thus, by definition, we know $\hat{x} \in \mathbb{K}^{\mathbb{S}}$, thus we have $\|\hat{x}\|_{Ni, 2}\leq \arg\min_{y \in \mathbf{K}^{{\mathbb{S}}}} \| y - x \|_{\mathcal{N}_i,2} \leq \tau^{*}$ and $\|\hat{x}\|_{Ni, \infty}\leq \arg\min_{y \in \mathbf{K}^{{\mathbb{S}}}} \| y - x \|_{\mathcal{N}_i,\infty} \leq \xi^{*}$ for all $\mathcal{N}_i \in \mathbb{S}$. On the other hand, we have
\begin{center}
    $\|x-\hat{x}\|^2_2 = \sum_{j=1, j\neq \hat{j}}^{p}(x_j - \hat{x}_j)^2 \leq \|x-\arg\min_{y \in \mathbf{K}^{{\mathbb{S}}}} \| y - x \|\|^2_2,$
\end{center}
which is a contradiction to the definition of $\arg\min_{y \in \mathbf{K}^{{\mathbb{S}}}} \| y - x \|_2$. Thus, the non-expansive coordinate-wise property is proved. \\
Now suppose $x\in \mathbb{K}^{\mathbb{S}}$, with $\mathbb{S}\in \mathbb{G}$. We can prove that 
\begin{center}
     $\arg\min_{y \in \mathbf{K}^{{\mathbb{G}}}} \| y - x \| = \arg\min_{y \in \mathbf{K}^{{\mathbb{G\setminus S}}}} \| y - x \|.$
\end{center}
Then we complete the proof since $\mathbb{S}$ can be any subset of $\mathbb{G}$. \\
First denote $v = \arg\min_{y \in \mathbf{K}^{{\mathbb{G\setminus S}}}} \| y - x \|$. By the non-expansive property, we get $\left| v_j\right| \leq \left| x_j\right|$ for all $j$ = 1,..., $p$. Therefore, $v \in \mathbf{K}^{{\mathbb{S}}} $. Also, by the assumptions stated above, we have $v \in \mathbf{K}^{{\mathbb{G\setminus S}}} $. Then we have $v \in \mathbf{K}^{{\mathbb{G}}} $.
\end{proof}

\subsection{Derivation of Finite Error Bound for Estimator of the SGLIG}
\label{pro: error}

\begin{lemma}
The regularizer $\Omega^{\mathbb{G}}_{p}$ is a norm.
\label{lemma:norm}
\end{lemma}
Lemma \ref{lemma:norm} can be proved by verifying the properties of a norm. A regularizer with double sparsity ($l_1$ and $l_2$) has been proven to be a norm by \citet{stephenson2019dsrig,rao2013sparse}. We denote two subspaces $\mathcal{M} \subseteq \overline{\mathcal{M}}$ of $\mathbb{R}^p$ and its orthogonal complement $\overline{\mathcal{M}}^\perp$ to ensure the decomposability of the regularizer, where $\overline{\mathcal{M}}^\perp$ is the set of all vectors in $\mathbb{R}^p$ that are perpendicular to vectors in $\overline{\mathcal{M}}$. We denote $\mathcal{M}$ as the model subspace to reflect the constraints specified by the SGLIG model. 

Assumption A3 assumes the true regression parameter is exactly sparse with $s$ non-zero components. Thus, the cardinality of the set of true predictors $J_0$ is $s$, and $s$ is smaller than the dimension $p$; i.e., $|J_0| = s < p$. Let $\mathcal{N}_i$ denote the neighborhood of predictor $i$, which includes the group information for predictor $i$. 

We can now split the model space into two subspaces based on the neighborhood information and the set of true predictors. With assumption A2, we know that the neighborhood of true predictors must also be true predictors. This means we can define a set $S_G$ to collect the neighborhood of true predictors; i.e., $\mathcal{N}_i \subseteq S_G$ if $i \subseteq J_0$. Specifically, we can define the subspaces using the group-structure norm example as introduced by \citet{negahban2012unified} under assumption A2. Since assumption A2 splits the neighborhoods into two non-overlapping parts, one part contains neighborhoods of true non-zero predictors, and the other part contains the neighborhoods of true zeros. We can define the subspaces as follows:
\begin{equation}
    \mathcal{M_{S_G}} := \left\{ \theta \in \mathbb{R}^p \mid \theta_{\mathcal{N}_i} = 0 \text{ for all } \mathcal{N}_i \notin S_G \right\},
\end{equation}
where $p$-dimensional subspace $M_{S_G}$ captures the true non-zero predictors from the model and 
\begin{equation}
    \overline{\mathcal{M}}_{S_G}^\perp := \left\{ \theta \in \mathbb{R}^p \mid \theta_{\mathcal{N}_i} = 0 \text{ for all } \mathcal{N}_i \in S_G \right\},
\end{equation}
where $p$-dimensional subspace $\overline{\mathcal{M}}_{S_G}^\perp$ captures the true zeros. Building on the definition of subspaces, we can now meet the condition of decomposability of the regularizer using Lemma \ref{lemma:decomposable}.

\begin{lemma}

Assume (A1)-(A3). Then the norm \( R(\beta) \) in the SGLIG model is decomposable with respect to the subspace pair $ \mathcal{M_{S_G}}, \,\overline{\mathcal{M}}_{S_G}^\perp $.
\label{lemma:decomposable}
\end{lemma}
It is straightforward to prove lemma \ref{lemma:decomposable}, since the components in $\mathcal{M_{S_G}}$ and $\overline{\mathcal{M}}_{S_G}^\perp$ are clearly non-overlapping under assumption A2. Given $R(\beta)$ is a decomposable norm, we can show that the dual norm of $R(\beta)$ is upper bounded.

Assumption A5 assumes that the loss function satisfies the RSC condition with curvature parameter $\kappa_L$. This assumption ensures that the loss difference among loss functions with true and estimated coefficients is close under the curvature parameter. In other words, the loss function is sufficiently convex so that a small loss difference leads to a small $\hat{\Delta}$. Since our loss function is differentiable, denote $\delta L(\Delta, \beta) := L(\beta + \Delta) - L(\beta) - \langle \nabla L(\beta), \Delta \rangle$ as the first order Taylor expansion that represents the approximation of $L$ in the direction of $\Delta$. \citet{negahban2012unified} summarized the definition of the RSC condition, and we extend it to our SGLIG model.

\begin{definition}
The loss function \( L(\beta) \) will satisfy a Restricted Strong Convexity (RSC) condition with curvature parameter \( \kappa_L \) if it is convex, differentiable, and
\[
\delta L(\Delta, \beta)\geq \kappa_L \| \Delta \|_2^2, \quad \forall \Delta \in C(\mathcal{M_{S_G}}, \overline{\mathcal{M}}_{S_G}^\perp, \beta),
\]
where \( \langle \cdot, \cdot \rangle \) represents the inner product of two vectors, and \( C(\cdot) \) is as defined in Equation (3.8).
\label{def:curvature}
\end{definition}
The definition summarized by \citet{negahban2012unified} includes a tolerance parameter. In the SGLIG model, the tolerance is equal to zero since $\hat{\beta} \in \mathcal{M_{S_G}}$ under assumption A2. Also, the compatibility constant is needed to ensure the regularizer $R(\beta)$ is limited, and we further find an upper bound for the compatibility constant with the SGLIG model (Lemma \ref{lemma: compati}), where the SGLIG model can be expressed as a standard regularized model with a decomposable estimator as follows:
\begin{equation}
\min_{\beta} \mathcal{L}(\beta) + \lambda \mathcal{R}(\beta), \,\,\beta = \sum_{i=1}^{p} V^{(i)}.
    \label{eq: penalized equation}
\end{equation}
Lemma \ref{lemma: compati} is defined as follows,
\begin{lemma}
Given a subspace \( \mathcal{M_{S_G}} \), the subspace compatibility constant with respect to a norm \( \| \cdot \| \) is given by
\[
\Psi(\mathcal{M_{S_G}}) = \sup_{u \in \mathcal{M_{S_G}} \setminus \{0\}} \frac{R(u)}{\| u \|}
\]
and the subspace compatibility constant associated with the optimization in Equation (\ref{eq: penalized equation}) is bounded by:
\[
\psi(M) \leq \left( \tau_{\max} + \sqrt{d_{\max}} \right) \sqrt{a},
\]
where $a$ is the number of non-zero $V^{(i)}$ in the optimal decomposition of $\beta$.
\label{lemma: compati}
\end{lemma}
\begin{proof}
    \begin{align*}
\\
R(\beta) 
&=  \sum_{i=1}^{p} \alpha\tau_i  \| V(i) \|_2 + (1-\alpha)\sqrt{d_i} \| V(i) \|_1  \\
&\leq  \sum_{i=1}^{p}\tau_{max}  \| V(i) \|_2 + \sqrt{d_{max}} \| V(i) \|_2    \\
&\leq  (\tau_{max} + \sqrt{d_{max}})\sum_{i=1}^{p}  \| V(i) \|_2    \\
&\leq  (\tau_{max} + \sqrt{d_{max}}) \sqrt{a}\  \| \beta \|_2.  
\end{align*}
\end{proof}

\begin{lemma}

Assume A4. Denote $u$ as a p-dimensional vector and $u_{\mathcal{N}_i}$ such that $(u_{\mathcal{N}_i})_j \neq 0 $ for $j \in \mathcal{N}_i$. Then the dual norm of \( R(\beta) \) in the SGLIG model is upper bounded as follows:
\[
R^*(u) \leq \max_{i = 1, \ldots, p}  \frac{1}{\sqrt{d^{min}}} \| u_G \|_2 .
\]
\label{lemma:upper bound of dual norm}
\end{lemma}
\begin{proof}
\begin{align*}
\\
R^*(u) &= \max_{\beta}  \{ \beta^Tu\} \quad\text{s.t.} \quad & R(\beta) \leq 1 \\
&= \max_V \left\{ \sum_{i=1}^{p} {V^{(i)}}^T u_{\mathcal{N}_i} \right\}  \quad\text{s.t.}  & \sum_{i=1}^{p} \alpha\tau_i  \|V^{(i)} \|_2 + (1-\alpha)\sqrt{d_i} \| V^{(i)} \|_1 \leq 1 \\
&\leq \max_V \left\{ \sum_{i=1}^{p} {V^{(i)}}^T u_{\mathcal{N}_i} \right\}  \quad\text{s.t.}  & \sum_{i=1}^{p} (\alpha\tau_i+\sqrt{d_i}-\alpha\sqrt{d_i}) \| V^{(i)} \|_2 \leq 1 \\
&\leq \max_V \left\{ \sum_{i=1}^{p} {V^{(i)}}^T u_{\mathcal{N}_i} \right\}  \quad\text{s.t.}  & \sum_{i=1}^{p} \sqrt{d_i}\| V^{(i)} \|_2 \leq 1 \\
&\leq \max_V \left\{ \sum_{i=1}^{p} {V^{(i)}}^T u_{\mathcal{N}_i} \right\}  \quad\text{s.t.}  & \sum_{i=1}^{p} \| V^{(i)} \|_2 \leq \frac{1}{\sqrt{d_{min}}}. \\
\end{align*}
Then the maximum is achieved by setting 
$V^{i^*} = \frac{u_{\mathcal{N}_i^*}}{\sqrt{d_{min}}\|u_{\mathcal{N}_i^*}\|_2}$ where $i^* = \arg\max \|u_{\mathcal{N}_i^*}\|_2$, then $R^*(u) \leq max_{1,\dots, p}\frac{1}{\sqrt{d_{min}}} \|u_{\mathcal{N}_i}\|_2 = \frac{1}{\sqrt{d_{min}}}  max_{1,\dots, p}\|u_{\mathcal{N}_i}\|_2$.\\
\end{proof}
The upper bound of the dual norm plays an important role in deriving the finite error bound of the SGLIG model. Now, denote the estimation error as $\hat{\Delta} = \hat{\beta} - \beta$. \citet{negahban2012unified} summarized that this estimation error $\hat{\Delta}$ belongs to a specific set using the decomposable norm $R(\beta)$. \citet{rao2013sparse,stephenson2019dsrig} also extended this summarization to a doubly sparse model. Furthermore, with the specific subspaces constructed for the SGLIG model, we can further show that $\hat{\Delta}$ belongs to a cone set. 

\begin{lemma}
    (\textbf{Proposition S.5 \citet{stephenson2019dsrig}}) For chi-square random variables \( X_1, X_2, \dots, X_p \) with \( d \) degrees of freedom and some constant \( c \), 

\[
P\left( \max_{i=1, \dots, p} X_i \leq c^2d \right) \geq 1 - \exp\left( \log(p) - \frac{(c - 1)^2 d}{2} \right).
\]
\label{lemma: prob}
\end{lemma}

\begin{lemma}
Assume (A4) and (A6). Then:

\[
R \left(\nabla L \right)^2 \leq \frac{\sigma^*\sigma^2 \left( \log(p) + d^{\max} \right)}{d^{\min}n}
\]
with probability at least \( 1 - c_1 \exp(-c_2 \sqrt{n}) \) for some \( c_1, c_2 > 0 \).
\end{lemma}

\begin{proof}
 \begin{align*}
     R^*(\nabla L) &\leq \max_{i = 1, \ldots, p}  \frac{1}{\sqrt{d^{min}}} \| \nabla L_{\mathcal{N}_i} \|_2 \\
     &= \frac{1}{n\sqrt{d^{min}}}\max_{i = 1, \ldots, p}\|  X_{\mathcal{N}_i}^T \epsilon \|_2,
 \end{align*}
where $\epsilon \sim \mathcal{N} \left( \mathbf{0}, \sigma^2 \mathbf{I} \right)$ and $X_{\mathcal{N}_i}^T \epsilon \sim \sigma\mathcal{N} \left( \mathbf{0}, X_{\mathcal{N}_i}^TX_{\mathcal{N}_i} \right)$. Let $\sigma_i^* $ be the maximum singular value of $ X_{\mathcal{N}_i}^TX_{\mathcal{N}_i}$. Then $\|  X_{\mathcal{N}_i}^T \epsilon \|_2^2 \leq \sigma^*\sigma^2\|  \mathcal{V}_i \|_2^2$ where $\mathcal{V}_i\sim \mathcal{N} \left( \mathbf{0}, d_i \mathbf{I} \right)$. Thus $\|  \mathcal{V}_i \|_2^2\sim \chi^2_{d_i}$, and

 \begin{align*}
     R^*(\nabla L)^2 
     &\leq \frac{1}{n^2d^{min}}\max_{i = 1, \ldots, p}\sigma^*\sigma^2\|  \mathcal{V}_i \|_2^2.
 \end{align*}
Then by Lemma \ref{lemma: prob},
 \begin{align*}
     P\left( \max_{i=1, \dots, p} \mathcal{V}_i \leq c^2d \right) \geq 1 - \exp\left( \log(p) - \frac{(c - 1)^2 d}{2} \right)
 \end{align*}
followed with
\begin{align*}
     R^*(\nabla L)^2 &\leq \frac{\sigma^*\sigma^2}{n^2d^{min}}\max_{i = 1, \ldots, p}\|  \mathcal{V}_i \|_2^2 \\
      &\leq \frac{\sigma^*\sigma^2}{n^2d^{min}} \frac{c^2d_{max}}{n}  \\
      \text{with probability} &\geq 1 - \exp\left( \log(p) - \frac{(c - 1)^2 d_{max}}{2} \right).
 \end{align*}
Then the upper bound of $R \left(\nabla L \right)^2$ is presented by setting $c = r\sqrt{n}$ and $r^2 = \frac{log(p) + d_{max}}{d_{max}}$. \citet{stephenson2019dsrig} provide explicit derivation of the constant $c_1$ and $c_2$ in the supplementary materials. 
\end{proof}

\begin{lemma}
Suppose \( L(\cdot) \) is a convex and differentiable loss function and consider any optimal solution \( \hat{\beta} \) to the optimization problem in Equation (\ref{eq: penalized equation}) with a strictly positive regularization parameter satisfying 
\[
\lambda \geq 2R^*(\nabla L(\beta)),
\]
where \( R^*(\cdot) \) is the dual norm of \( R(\cdot) \) and \( \nabla L(\beta) \) is the gradient of the loss function. Assume (A1)-(A3), and let \( \Pi_\mathcal{M_{S_G}}(\cdot) \) represent the projection onto the subspace \( \mathcal{M_{S_G}} \). Then, the error, \( \hat{\Delta} = \hat{\beta} - \beta \), will belong to the set:
\[
C(\mathcal{M_{S_G}}, \overline{\mathcal{M}}_{S_G}^\perp, \beta) := \left\{ \Delta \in \mathbb{R}^p \mid R(\Pi_{\overline{\mathcal{M}}_{S_G}^\perp} \Delta) \leq 3R[\Pi_\mathcal{M_{S_G}}(\Delta)] \right\}.
\]
\label{lemma:cone}
\end{lemma}
Lemma \ref{lemma:upper bound of dual norm} states that there is an upper bound for the dual norm of $R(\beta)$, so there always exists a $\lambda$ that is greater than $2R^*(\nabla L(\beta))$. Then, Lemma \ref{lemma:cone} guarantees that $\hat{\Delta}$ is confined within a specific cone, which in turn ensures the existence of an upper bound for $\hat{\Delta}$ for any unknown $\hat{\beta}$.

Suppose Lemma \ref{lemma:decomposable} and Definition \ref{def:curvature} hold. \citet{negahban2012unified} proposed bounds for general models in the following. 
\begin{theorem}  (\citet{negahban2012unified}): Consider the loss function $L$ is convex and differentiable that meets the RSC condition with curvature $ \kappa_L$, and the regularizer $R(\beta)$ is a decomposable norm with respect to subspace pair $ \mathcal{M_{S_G}}, \,\overline{\mathcal{M}}_{S_G}^\perp $. Then, for the optimization problem (\ref{eq: penalized equation}), given \( \lambda \geq 2R^*(\nabla L(\beta)) \), any optimal solution $\hat{\beta}$ satisfies 

\[
\| \hat{\beta}_{\lambda} - \beta \|_2^2 \leq 9 \lambda^2 \frac{\Psi^2(\mathcal{M_{S_G}})}{\kappa}.
\]

\label{th: negban}
\end{theorem} 

\subsection{Asymptotic Normality of Estimator for the SGLIG}
\label{pro: asym}
\textbf{Proof of Asymptotic Normality of Estimator for the SGLIG:}
The proof is made based on the supplementary materials from \citet{yu2016sparse}.
\\For each \( u \in \mathbb{R}^p \), define
\[
Q_n(u) = \frac{1}{2} \left\| \frac{1}{\sqrt{n}} X u - \epsilon \right\|_2^2 + n\lambda \| \beta^0 + \frac{u}{\sqrt{n}}\|_{G,\tau} .
\]
It is easy to check that
\[
\hat{u} = \sqrt{n}(\hat{\beta} - \beta^0) = \arg \min_{u \in \mathbb{R}^p} Q_n(u).
\]
Furthermore, we have
\[
Q_n(u) - Q_n(0) = \frac{1}{2} \left\| \frac{1}{\sqrt{n}} X u - \epsilon \right\|_2^2 + n\lambda \| \beta^0 + \frac{u}{\sqrt{n}}\|_{G,\tau} - \frac{1}{2} \| \epsilon \|_2^2 - n\lambda \| \beta^0 \|_{G,\tau}
.\]
This simplifies to:
\[
Q_n(u) - Q_n(0) = \underbrace{\frac{1}{2n} u^T X^T X u - \frac{1}{\sqrt{n}} u^T X^T \epsilon}_{I_1} + \underbrace{n\lambda (\| \beta^0 + \frac{u}{\sqrt{n}}\|_{G,\tau} -  \| \beta^0 \|_{G,\tau})}_{I_2}.
\]
By assumptions (A4) and (A5), we get:
\[
I_1 \overset{d}{\to} \frac{1}{2} u^T M u - u^T W,
\]
where \( W \sim N_p(0, \sigma^2 M) \).

Now assume first \( |J_0| \) elements of \( \beta \) are nonzero, and the other \( p - |J_0| \) elements are zero; i.e., 
\[
\beta^0 = \begin{pmatrix} \beta^0_{J_0} \\ 0 \end{pmatrix}.
\]
Hence,
\[
I_2 = n\lambda (\left\| \begin{pmatrix} \beta^0_{J_0} + \frac{1}{\sqrt{n}} u_{J_0} \\ \frac{1}{\sqrt{n}} u^{J_c} \end{pmatrix} \right\|_{G,\tau}
-  \left\| \begin{pmatrix} \beta^0_{J_0} \\ 0 \end{pmatrix} \right\|_{G,\tau}).
\]
Further,
\[
I_2 = \underbrace{n\lambda (\left\| \begin{pmatrix} \beta^0_{J_0} + \frac{1}{\sqrt{n}} u_{J_0} \\ 0 \end{pmatrix} \right\|_{G,\tau}
-  \left\| \begin{pmatrix} \beta^0_{J_0} \\ 0 \end{pmatrix} \right\|_{G,\tau})}_{I_3}+ \underbrace{\sqrt{n}\lambda (\left\| \begin{pmatrix} 0 \\ u_{j_0} \end{pmatrix} \right\|_{G,\tau})}_{I_4}.
\]
Denote \( V^{(1)}, V^{(2)}, \dots, V^{(p)} \) as an arbitrary optimal decomposition of \( u \). Also, by the triangle inequality:
\[
|I_3| \leq n\lambda \left\| \begin{pmatrix} \frac{1}{\sqrt{n}} u^{J_0} \\ 0 \end{pmatrix} \right\|_G = \sqrt{n}\lambda\{\sum_{j \in J_0} \alpha\tau_j\| V^{(j)} \|_2 + (1-\alpha)\sqrt{d_j}\| V^{(j)} \|_1\}.
\]
Assume \( \sqrt{n}\lambda \to 0 \) and \( |\tau_{j}| = O(1) \) for each \( j \in J_0 \), given $0<\alpha<1$ and $d_j$ is bounded. Then for each fixed \( u \), we have \( |I_3| \to 0 \) as \( n \to \infty \).
Furthermore, we observe that
\[
|I_4| = \sqrt{n}\lambda \sum_{j \in J_c^0} \alpha\tau_j\| V^{(j)} \|_2 + (1-\alpha)\sqrt{d_j}\| V^{(j)} \|_1 = (n^{\frac{\gamma+1}{2}}\lambda)(n^{\frac{-\gamma}{2}} \sum_{j \in J_c^0} \alpha\tau_j\| V^{(j)} \|_2 + (1-\alpha)\sqrt{d_j}\| V^{(j)} \|_1).
\]
Assume \(   n^{\frac{\gamma+1}{2}}\lambda  \to \infty \) for some $\gamma>0$, \( \liminf_{n \to \infty} n^{\frac{-\gamma}{2}}\alpha\tau_j > 0 \) and \( \liminf_{n \to \infty} n^{\frac{-\gamma}{2}}(1-\alpha)\sqrt{d_j} > 0 \).
Then we have \( |I_4| \to \infty \) as \( n \to \infty \).\\
Hence, we can summarize that :
\[
Q_n(u) - Q_n(0) = I_1 + I_2+I_3,
\]
where 
\[
I_1 \overset{d}{\to} \frac{1}{2} u^T M u - u^T W,
\]
while $I_3 + I_4\overset{d}{\to}0$ if $supp(u)\subseteq J_0$ and $I_3 + I_4\overset{d}{\to}\infty $ otherwise.
Therefore, we can conclude that
\[
Q_n(u) - Q_n(0) \overset{d}{\to} D(u),
\]
where
\[
D(u) =
\begin{cases}
\frac{1}{2} u^T M u - u^T W, & \text{if } \text{supp}(u) \subseteq J_0, \\
\infty, & \text{otherwise}.
\end{cases}
\]
Since \( \hat{u} = \arg \min_{u \in \mathbb{R}^p} Q_n(u) = \arg \min_{u \in \mathbb{R}^p} (Q_n(u) - Q_n(0)) \) , and let
\[
u^* = \begin{pmatrix} M^{-1}_{J_0, J_0} W_{J_0} \\ 0 \end{pmatrix} = \arg \min_{u \in \mathbb{R}^p} D(u)
\]
be the minimizer so that  $\hat{u}\overset{d}{\to}u^*$, by the \textbf{Argmax Theorem} (\citet{van1996weak}, Corollary 3.2.3). We have:
\[
\hat{u} \overset{d}{\to} (M^{-1}_{J_0, J_0} W_{J_0},0)^T.
\]
Thus,
\[
\sqrt{n} (\hat{\beta}_{J_0} - \beta^0_{J_0}) \overset{d}{\to} N(0, \sigma^2 M^{-1}_{J_0, J_0}) \text{ for} j\in J_0
\]
and 
\[
\sqrt{n} \hat{\beta}_{J_c^0} \overset{d}{\to} 0 \text{ therefore } \hat{\beta}_{J_c^0} \overset{d}{\to} 0.
\]

\subsection{Comparisons of Finite Error Bound Between the SGLIG and DSRIG}
\label{pro: diffbound}
\subsubsection*{Scenario 1: Comparisons Under Assumptions of DSRIG}
The DSRIG uses the adaptive weights for the $l_2$ penalty defined as $\tau_i = \frac{\sqrt{d_i}}{cov(X, Y_i)}$, where $\left|cov(X, Y_i) \leq1\right|$ for $\forall i \in \{1, \dots, p\}$. Therefore, we have $\tau_i^2 \geq \sqrt{d_i}$ for $\forall i \in \{1, \dots, p\}$. Then it can be shown that $\tau_{min}^2 \geq d_{min}$ and $\frac{1}{\tau_{min}^2} \leq \frac{1}{d_{min}}$. With this, we obtain
\begin{equation}
\frac{1}{\tau_{min}^2} \leq \frac{4}{d_{min}}.
\label{con: 1 for DSRIG}
\end{equation}

\citet{stephenson2019dsrig} create an assumption A4 that sets $\xi$ to be lower bounded by 1. However, this assumption is usually violated during the computation of the DSRIG model on both simulated and real-world datasets. Therefore, with the violation of assumption A4 from DSRIG, i.e., $0<\xi\leq1$, it can be shown that
\begin{equation}
    \left( \tau_{\max} +  \xi\sqrt{d_{\max}} \right)^2 \leq \left( \tau_{\max} +  \sqrt{d_{\max}} \right)^2.
    \label{con: 2 for DSRIG}
\end{equation}
Now, based on equations \ref{con: 1 for DSRIG} and \ref{con: 2 for DSRIG}, we can conclude that the DSRIG model has a smaller finite error bound compared to the SGLI model. 

\subsubsection*{Scenario 2: Comparisons Under Assumptions of SGLIG}
Without the violation of assumption A4, the DSRIG model sets $\xi\geq1$ during the derivation of the finite error bound. Therefore, 
\begin{equation}
    \left( \tau_{\max} +  \sqrt{d_{\max}} \right)^2 \leq \left( \tau_{\max} +  \xi\sqrt{d_{\max}} \right)^2
    \label{con: 2 for SGLIG}
\end{equation}
for $\xi\geq1$. 
The SGLIG model assumes $\tau_i$ is lower bounded by \( \sqrt{d_i} \) for \( i = 1, \dots, p \) in assumption A4., which means $\tau_{min}^2 \geq d_{min}$. However, The SGLIG model uses different adaptive weights for the $l_2$ norm which is defined as $\tau_i = \frac{1}{cov(X, Y_i)}$ for $\forall i \in \{1, \dots, p\}$, so there is no direct scale condition between $\tau_{min}$ and $\sqrt{d_{min}}$ when fitting the SGLIG model to the simulated and real-world data. Thus, the violation of assumption A4 is interpretable. Therefore, with the violation of assumption A4 from the SGLIG model, we assume $\sqrt{d_{min}}\geq2\tau_{min}$, and it can be shown that
\begin{equation}
\frac{4}{d_{min}} \leq \frac{1}{\tau_{min}^2}.
\label{con: 1 for SGLIG}
\end{equation}
Based on equations \ref{con: 2 for SGLIG} and \ref{con: 1 for SGLIG}, we can conclude that the SGLIG model has a smaller finite error bound compared to the DSRIG model.

\section{Appendix B: Alzheimer’s Disease Neuroimaging Initiative}

\subsection{Alzheimer’s Disease Neuroimaging Initiative}
We fit the graphical models to the Alzheimer’s Disease Neuroimaging Initiative (ADNI) databases (\url{adni.loni.usc.edu}), hosted by the Laboratory of Neuro Imaging (LONI) at the University of Southern California. The ADNI was launched in 2003 as a public-private
partnership, led by Principal Investigator Michael W. Weiner, MD. The primary goal of ADNI has been to
test whether serial magnetic resonance imaging (MRI), positron emission tomography (PET), other
biological markers, and clinical and neuropsychological assessment can be combined to measure the
progression of mild cognitive impairment (MCI) and early Alzheimer’s disease (AD). The ADNI data aims to track the progression of Alzheimer’s Disease. Alzheimer’s Disease is the most common neurodegenerative disease, causing dementia, characterized by progressive cognitive and memory deficits \citep{yu2016sparse, leeuwis2017lower}. \citet{folstein1975mini} proposed the Mini Mental State Examination (MMSE) as a tool for evaluating cognitive impairment and dementia situations. MMSE has currently become a critical reference in the diagnosis of Alzheimer’s Disease. Concretely, MMSE is a 30-point questionnaire that includes 11 questions assessing five areas of cognitive function independently. A score of 27 or higher is typically considered indicative of normal cognitive function.

The SRIG and DSRIG models were applied to ADNI data, with MMSE score used as the response variable \citep{yu2016sparse, stephenson2019dsrig}. \cite{yu2016sparse} used 93 manually labeled regions of interest (ROI)\citep{zhang2012multi} and \cite{stephenson2019dsrig} used 135 volume measurements ($mm^3$) obtained from the Freesurfer segmented data as the predictors. FreeSurfer is an image analysis suite used to process the MRI data collected in ADNI. It was developed by \citet{hartig2014ucsf} and is available online (http://surfer.nmr.mgh.harvard.edu/). We are also interested in predicting the MMSE scores using MRI features and exploring the underlying graphical structure among these features. ADNI data collects the MRI images and transforms them into numerical MRI features after some image pre-processing steps.

We downloaded the MR image analysis dataset, last updated on March 31, 2024, from the Image and Data Archive(ida.loni.usc.edu). This dataset includes three phases of ADNI data: ADNI 1 (2004–2010), ADNI GO (2009–2011), and ADNI 2 (2011–2016). Each phase of ADNI has distinct primary goals, with ADNI 2 data focusing on developing biomarkers as predictors of cognitive decline. For our analysis, we excluded ADNI 1 and ADNI GO data, retaining only the ADNI 2 data collected 6 months after the baseline visit. From the MR image features, we selected 119 volume measurements (in mm3), which were processed using FreeSurfer. We removed subjects who failed the overall MR image quality control or had missing data in the volume measurement features. After these exclusions, the dataset consists of  226 subjects and 119 brain volume features for further analysis. For details on the data splitting and training procedure, see the real-data analysis section of the main manuscript. The corresponding MMSE scores for each subject were obtained from the MMSE document downloaded online, and these brain volume features were used as predictors to model the MMSE scores.

\citet{yu2016sparse} evaluate their SRIG model by comparing it to LASSO, ridge, Adaptive LASSO, Elastic net, GOSCAR, GRCE, PCR, and SPLS \citep{yang2012feature,li2008network,wehrens2007pls}. They have shown that the SRIG model outperforms others on the ADNI data, so we only need to compare our model with the SRIG model. For each model (SRIG, SGLIG, and DSRIG), we record the number of selected nonzero predictors and the number of estimated edges. Since the same tuning parameters are used in the MB method, all models yield the same number of estimated edges. Model performance is evaluated in terms of MSE and computational time. From Table \ref{tab: adni}, we observe that SGLIG achieves the best MSE performance, while SRIG has the highest MSE score. The SGLIG is much more computationally efficient than the DSRIG, showing better MSE performance while requiring much less time. In contrast, the DSRIG model spends ten times more time than both the SRIG and the SGLIG models, but still has a higher MSE score than the SGLIG model.
\begin{table}[ht]
\caption{Performance of SRIG, SGLIG, DSRIG on the ADNI data in terms of two key metrics: mean square error, and computation time (seconds). }
\centering
\begin{tabular}{p{2cm} c S[table-format=1.1]cp{1cm}S[table-format=2.2]}
\toprule
Model & Predictors & {Edge} & non-zero & {MSE} & {Time} \\
\midrule
SRIG  & 119 & 72.74 & 7.60 & 0.734 & 4.09 \\
SGLIG & 119 & 72.74 & 5.36 & 0.722 & 5.91 \\
DSRIG & 119 & 72.74 & 5.36 & 0.728 & 64.00 \\
\bottomrule
\end{tabular}
\label{tab: adni}
\end{table}
\\Figure \ref{fig: adni graph} compares the MSE differences between SRIG, SGLIG, and DSRIG across all 90 permutations of the ADNI dataset. Both SRIG-SGLIG and DSRIG-SGLIG plots reveal that the SGLIG model outperforms the other models in most permutations of ADNI data. In summary, the SGLIG model outperforms both the DSRIG and SRIG models on ADNI data, providing better prediction accuracy with smaller computational resources.
\begin{figure}[h!]
    \centering
    \includegraphics[width=1\textwidth]{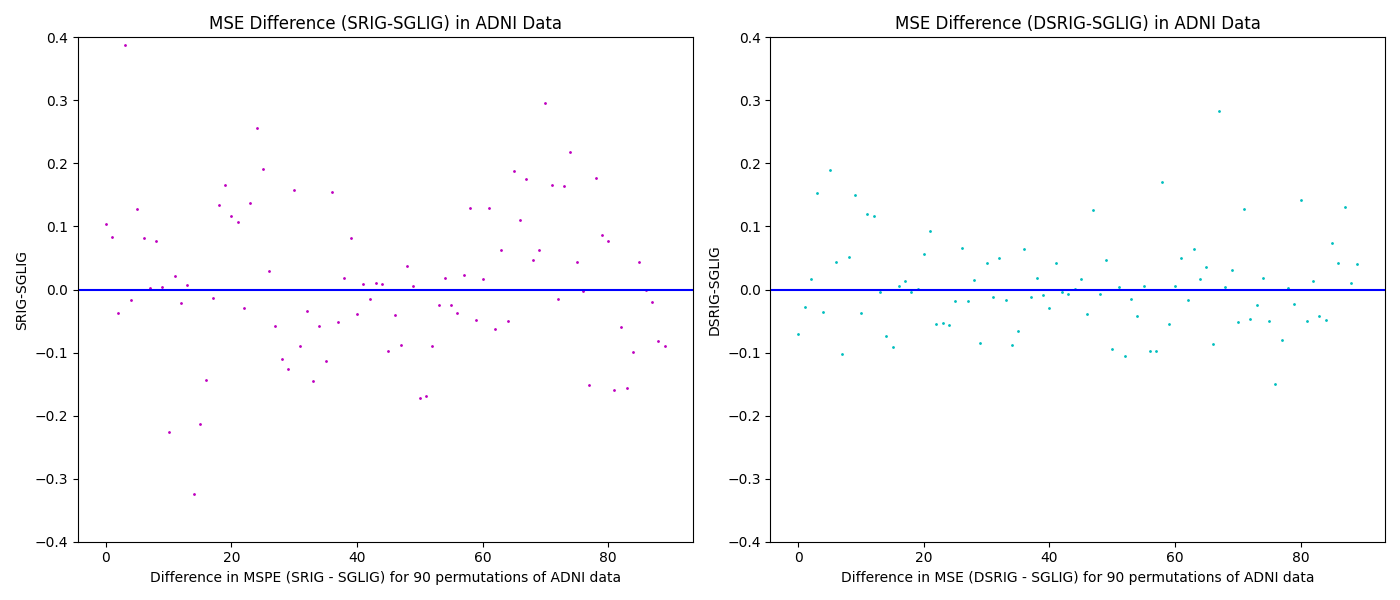} 
    \caption{MSE for 90 permutations of ADNI data}
    \label{fig: adni graph}
\end{figure}

Table \ref{table: consis adni} lists the variables consistently selected in over 70 out of 90 permutations for the ADNI data. A similar pattern of consistent variable selection is observed between the SGLIG and DSRIG models. 
\begin{table}[h!]
\caption{Selected variables that were selected by over 70 out of 90 permutations}
\centering
\begin{tabular}{lcccc}
\hline
           &           & \multicolumn{3}{c}{Selected Variables} \\ \cline{3-5} 
          Variable                &   Brain Region          & SRIG & SGLIG & DSRIG \\ \hline
ST12SV                    & Left Amygdala & \checkmark    & \checkmark     & \checkmark     \\ 
ST24CV                    & Left Entorhinal & \checkmark    & -     & -     \\ 
ST31CV                    & Left Inferior Parietal & \checkmark    & \checkmark     & \checkmark     \\ 
ST32CV                    & Left Inferior Temporal & \checkmark    & \checkmark     & \checkmark     \\ 
ST37SV                    & Left Lateral Ventricle & \checkmark    & -     & -     \\ 
ST110CV                   & Right Precentral & -    & \checkmark     & \checkmark     \\ 
ST29SV                    & Left Hippocampus & -    & \checkmark     & \checkmark     \\ 
ST30SV                    & Left Inferior Lateral Ventricle & -    & \checkmark     & -     \\ 
ST57CV                    & Left Superior Parietal & -    & \checkmark     & -     \\ 
ST58CV                    & Left Superior Temporal & -    & \checkmark     & \checkmark     \\ 
ST26CV                    & Left Fusiform & -    & -     & \checkmark     \\ \hline
\end{tabular}
\label{table: consis adni}
\end{table}
Most of the variables selected by the SRIG model are also selected by the SGLIG and DSRIG. Among the brain regions selected by all three graphical models, almost all are located on the left except for the right precentral. In addition, the selected regions, including the amygdala and hippocampus, are known to be associated with AD in previous studies \citep{ jack1999prediction, misra2009baseline, zhang2012multi, yu2016sparse, stouffer2024amidst, prabhu2024abnormal}.
\clearpage

\end{document}